\documentclass[smallextended]{svjour3}  

\usepackage{amssymb}
\usepackage{amsmath}
\usepackage{algorithm}
\usepackage{graphicx, graphics}
\usepackage{epstopdf}
\usepackage{subfigure}
\usepackage{mathtools}
\usepackage{color}
\usepackage{booktabs}       % professional-quality tables
\usepackage{amsfonts}       % blackboard math symbols
\usepackage{nicefrac}       % compact symbols for 1/2, etc.
\usepackage{hyperref}
\usepackage[round]{natbib}
\usepackage{ragged2e}
\usepackage{multirow}

\newcommand{\rank}{\mathrm{rank}}

\newcommand{\Stiefel}[2]{{\mathrm{St}({#1},{#2})}}
\newcommand{\Grass}[2]{{\mathrm{Gr}({#1},{#2})}}

\newcommand{\OG}[1]{{\mathcal{O}({#1})}}

\newcommand{\argmin}{\operatornamewithlimits{arg\,min}}

\newcommand{\grad}{\mathrm{grad}}

\newcommand{\mat}[1]{{\bf #1}}

\newcommand{\Grad}{\mathrm{Grad}}
\newcommand{\Exp}{{\mathrm{Exp}}}
\newcommand{\Log}{{\mathrm{Log}}}
\renewcommand{\cos}{{\mathrm{cos}}}
\renewcommand{\sin}{{\mathrm{sin}}}

\newcommand{\noteBM}[1]{#1}

\newcommand{\changeBMM}[1]{#1}

\begin{document}

%\begin{frontmatter}

\title{A Riemannian gossip approach to\\ subspace learning on Grassmann manifold}



\author{Bamdev Mishra         \and
        Hiroyuki Kasai \and
        Pratik Jawanpuria \and
        Atul Saroop
}


\institute{Bamdev Mishra \at
              Amazon.com \\
              %Bengaluru\\
              \email{bamdevm@amazon.com}           %  \\
           \and
           Hiroyuki Kasai \at
           The University of Electro-Communications\\
           %Tokyo\\
           \email{kasai@is.uec.ac.jp}
           \and
           Pratik Jawanpuria \at
           Amazon.com \\
           %Bengaluru\\
           \email{jawanpur@amazon.com}
           \and
           Atul Saroop \at
           Amazon.com \\
           %Bengaluru\\
           \email{asaroop@amazon.com}
}

\date{Received: date / Accepted: date}

\maketitle

\begin{abstract}
%In this paper, we focus on subspace learning problems that are modeled as finite sum problems on the Grassmann manifold. Interesting applications in this setting include low-rank matrix completion and  multivariate regression, both of which are naturally modeled in the considered setup. Motivated by privacy concerns, we aim to solve such problems in a decentralized setting and to this end propose novel gossip algorithms.
%
%In order to exploit the finite sum structure in a decentralized setting, the problem is distributed among different {\it agents}.  
%This is facilitated by a novel cost function, which is a weighted sum of several sub-problems (each solved by an agent) and the communication cost among the agents. Communication among the agents is required to arrive at consensus. In the proposed modeling approach, different agents learn individual local subspace but they achieve asymptotic consensus on the global learned subspace.
%
%The proposed approach is scalable and parallelizable. Our numerical experiments on various real-world datasets show the good performance of the proposed decentralized algorithms on various matrix completion and multivariate regression  benchmarks.

In this paper, we focus on subspace learning problems on the Grassmann manifold. Interesting applications in this setting include low-rank matrix completion and low-dimensional multivariate regression, among others. Motivated by privacy concerns, we aim to solve such problems in a decentralized setting where multiple agents have access to (and solve) only a part of the whole optimization problem. The  agents communicate with each other to arrive at a consensus, i.e., agree on a common quantity, via the gossip protocol.

We propose a novel cost function for subspace learning on the Grassmann manifold, which is a weighted sum of several sub-problems (each solved by an agent) and the communication cost among the agents. The cost function has a finite sum structure. In the proposed modeling approach, different agents learn individual local subspace but they achieve asymptotic consensus on the global learned subspace. The approach is scalable and parallelizable. Numerical experiments show the efficacy of the proposed decentralized algorithms on various matrix completion and multivariate regression benchmarks.

%[Original abstract] In this paper, we propose novel gossip algorithms for decentralized subspace learning problems that are modeled as finite sum problems on the Grassmann manifold. Interesting applications in this setting include low-rank matrix completion and multitask feature learning, both of which are naturally reformulated in the considered setup. To exploit the finite sum structure in a decentralized setting, the problem is distributed among different agents and a novel cost function is proposed which is a weighted sum of the tasks handled by the agents and the communication cost among the agents. The proposed modeling approach allows local subspace learning by different agents while achieving asymptotic consensus on the global learned subspace. The resulting approach is scalable and parallelizable. Our numerical experiments on various real datasets show the good performance of the proposed decentralized algorithms on various matrix completion and multitask learning benchmarks.
\end{abstract}

%\begin{keyword}
%Non-linear gossip\sep stochastic gradients\sep manifold optimization\sep matrix completion\sep multivariate regression
%\MSC[2010] 00-01\sep  99-00
%\end{keyword}

%\end{frontmatter}

%\linenumbers

\section{Introduction}
Learning a low-dimensional representation of vast amounts of data is a fundamental problem in machine learning. It is motivated by considerations of low memory footprint or low computational complexity, model compression, better generalization performance, robustness to noise, among others. The applicability of low-dimensional modeling is ubiquitous, including images in computer vision, text documents in natural language processing, genomics data in bioinformatics, and customers' record or purchase history in recommender systems. 

Principal component analysis (PCA) is one of the most well known algorithms employed for low-dimensional representation in data analysis \citep{Bishop06}. PCA is employed to learn a low-dimensional subspace that captures the most variability in the given data. Collaborative filtering based applications, such as movie or product recommendation, desire learning a latent low-dimensional subspace that captures users' preferences~\citep{rennie05a,zhou08a,abernethy09a}. The underlying assumption here is that similar users have similar preferences.  A common approach to model this problem is via low-rank matrix completion: recovering low-rank matrices when most entries are unknown \citep{candes09b,cai10a,wen12a}. Motivated by similar requirements of learning a low-dimensional subspace, low-rank matrix completion algorithms are also employed in other applications such as system identification \citep{markovsky13a}, subspace identification \citep{balzano10a}, sensor networks \citep{keshavan09a}, and gene expression prediction \citep{kapur16a}, to name a few. 

In several multivariate regression problems, we need to learn the model parameters for several {\it related} regression tasks (problems), but the amount of labeled data available for each task is low. In such data scarce regime, learning each regression problem (task) only with its own labeled data may not give good enough generalization performance~\citep{baxter97,baxter00,Jalali10,Alvarez12,zhang17a}. The paradigm of multitask learning~\citep{Caruana97} advocates learning these related tasks {\it jointly}, i.e., each tasks not only learns from its own labeled data but also from the labeled data of other tasks. Multitask learning is helpful when the tasks are {\it related}, e.g., the model parameters of all the tasks have some common characteristics that may be exploited during the learning phase. 
Existing multitask literature have explored various ways of learning the tasks jointly~\citep{evgeniou04a,Jacob08,zhang10a,Zhong12,mjaw12,Kumar12b,zhang15a}. Enforcing the model parameters of all the tasks to share a common low-dimensional latent feature space is a common approach in multitask (feature) learning~\citep{ando05a,amit07a,argyriou08a}. 

%Learning a low-dimensional subspace is a fundamental problem in many machine learning applications, e.g., in system identification \citep{markovsky13a}, subspace identification \citep{balzano10a}, collaborative filtering \citep{rennie05a,meyer11c}, multitask learning \citep{argyriou08a,amit07a,ando05a}, and information theory \cite{shi16a}, to name a few. The set of $r$-dimensional subspaces in $\mathbb{R}^m$ has the structure of a smooth and compact manifold, known as the \emph{Grassmann} manifold $\Grass{r}{m}$ \citep{edelman98a,absil08a}. 

A low-dimensional subspace can be viewed as an instance of the \emph{Grassmann} manifold $\Grass{r}{m}$, which is the set of $r$-dimensional subspaces in $\mathbb{R}^m$. A number of Grassmann algorithms exploiting the geometry of the search space exist for subspace learning, in both batch \citep{absil08a} and online variants \citep{bonnabel13a,zhang16a,sato17a}. Several works  \citep{balzano10a,dai11a,he12a,boumal15a} discuss a subspace learning approach based on the Grassmann geometry from incomplete data. The works \citep{meyer09a,meyer11c} exploit the Grassmann geometry in distance learning problems through low-rank subspace learning. More recently, the works \citep{harandi16a,harandi17a,harandi18a} show the benefit of the Grassmann geometry in low-dimensional dictionary and metric learning problems.

In this paper, we are interested in a \emph{decentralized} learning setting on the Grassmann manifold, which is less explored for the considered class of problems. To this end, we assume that the given  data is {\it distributed} across several agents, e.g., different computer systems. The agents can learn a low-dimensional subspace only from the data that resides locally within them, and cannot access data residing within other agents. This scenario is common in situations where there are privacy concerns of sharing sensitive data. The works \citep{ling12a,lin15a} discuss decentralized algorithms for the problem of low-rank matrix completion. The agents can communicate with each other to develop consensus over a required objective, which in our case, is the low-dimensional subspace. The communication between agents causes additional computational overheads, and hence should ideally be as little as possible. This, in addition to privacy concerns, motivate us to employ the so-called \emph{gossip} protocol in our setting \citep{boyd06a,shah09a,colin16a}. In the gossip framework, an agent communicates with only one other agent at a time \citep{boyd06a}. 

Recently, \citet[Section~4.4]{bonnabel13a} discusses a non-linear gossip algorithm for estimating covariance matrix $\mat{W}$ on a sensor network of multiple agents. Each agent is initialized with a local covariance matrix estimate, and the aim there is to reach a common (average) covariance matrix estimate via communication among the agents. If $\mat{W}_i$ is the estimate of the covariance matrix possessed by agent $i$, \citet{bonnabel13a} proposes to minimize the cost function $$\sum_{i=1}^{m-1} d^2(\mat{W}_i,\mat{W}_{i+1}),$$ to arrive at consensus, where the total number of agents is $m$ and $d$ is a distance function between covariance matrix estimates. At each time slot, a randomly chosen agent $i(<m)$ communicates with its neighbor agent $i+1$ and both update their covariance matrix estimates. \citet{bonnabel13a} shows that under mild assumptions, the agents converge to a common covariance matrix estimate, i.e., the agents achieve consensus. It should be noted that consensus learning on manifolds has been in general a topic of much research, e.g., the works \citep{sarlette09a,tron11a,tron13a} study the dynamics of agents which share their relative states over a more complex communication graph (than the one in \cite[Section~4.4]{bonnabel13a}) . The aim in \citep{sarlette09a,tron11a,tron13a,bonnabel13a} is to make the agents converge to a single point. In this paper, however, we dwell on consensus learning of agents along with optimizing the sub-problems handled by the agents. For example, at every time instance a randomly chosen agent locally updates its local subspace (e.g., with a gradient update) and simultaneously communicate with its neighbor to build a consensus on the global subspace. This is a typical set up encountered in machine learning based applications. The paper does not aim at a comprehensive treatment of consensus algorithms on manifolds, but rather focuses on the role of the Grassmann geometry in coming out with a simple cost problem formulation for decentralized subspace learning problems. %To this end, we exploit the agent network architecture of \cite[Section~4.4]{bonnabel13a}.

We propose a novel optimization formulation on the Grassmann manifold that combines together a weighted sum of tasks (accomplished by agents individually) and consensus terms (that couples subspace information transfer among agents). The weighted formulation allows an {\it implicit} averaging of agents at every time slot. The formulation allows to readily propose a stochastic gradient algorithm on the Grassmann manifold and further allows a {\it parallel} implementation (via a modified sampling strategy). For dealing with ill-conditioned data, we also propose a \emph{preconditioned variant}, which is computationally efficient to implement. We apply the proposed approach on two popular subspace learning problems: low-rank matrix completion \citep{cai10a,keshavan10a,balzano10a,boumal15a,boumal11a,dai11a} and multitask feature learning \citep{ando05a,argyriou08a,Zhang08,zhang17a}. Empirically, the proposed algorithms compete effectively with state-of-the-art on various benchmarks. 

The organization of the paper is as follows. Section \ref{sec:grassmann} presents a discussion on the Grassmann manifold. Both low-rank matrix completion and multitask feature learning problems are motivated in Section \ref{sec:motivation} as finite sum problems on the Grassmann manifold. In Section \ref{sec:decentralized_problem_setup}, we discuss the decentralized learning  setup and propose a novel problem formulation. In Section \ref{sec:gossip}, we discuss the proposed stochastic gradient based gossip algorithm along with preconditioned and parallel variants. Experimental results are discussed in Section \ref{sec:comparisons}. The present paper extends the unpublished technical report \citep{mishra16c}. The Matlab codes for the proposed algorithms are available at \url{https://www.bamdevmishra.com/gossip}.

\section{Grassmann manifold}\label{sec:grassmann}
The Grassmann manifold $\Grass{r}{m}$ is the set of $r$-dimensional subspaces in $\mathbb{R}^m$. In matrix representation, an element of $\Grass{r}{m}$ is represented by the \emph{column space} of a full rank matrix of size $m\times r$. Equivalently, if $\mat{U}$ is a full rank matrix of size $m \times r$, an element of $\Grass{r}{m}$ is represented as 
\begin{equation}\label{eq:col_space}
\mathcal{U} \coloneqq \text{the column space of }\mat{U}.
\end{equation}
Without loss of generality, we impose \emph{orthogonality} on $\mat U$, i.e., $\mat{U}^\top{\mat U} = \mat{I}$. This characterizes the columns space in (\ref{eq:col_space}) and allows to represent $\mathcal{U}$ as follows:
\begin{equation}\label{eq:rotational_invariance}
\mathcal{U} \coloneqq \{ \mat{\mat U O}:  \mat{O} \in \OG{r}\},
\end{equation}
where $\OG{r}$ denotes the orthogonal group, i.e., the set of $r\times r$ orthogonal matrices. An implication of (\ref{eq:rotational_invariance}) is that each element of $\Grass{r}{m}$ is an \emph{equivalence set}. This allows the Grassmann manifold to be treated as a \emph{quotient} space of the larger Stiefel manifold $\Stiefel{r}{m}$, which is the set of matrices of size $m\times r$ with \emph{orthonormal} columns. Specifically, the Grassmann manifold has the \emph{quotient manifold} structure
\begin{equation}\label{eq:quotient_structure}
\Grass{r}{m} \coloneqq \Stiefel{r}{m}/\OG{r}.
\end{equation}

A popular approach to optimization on a quotient manifold is to recast it to into a Riemannian optimization framework \citep{edelman98a,absil08a}. In this setup, while optimization is conceptually on the Grassmann manifold $\Grass{r}{m} $, numerically, it allows to implement operations with concrete matrices, i.e., with elements of $\Stiefel{r}{m}$. Geometric objects on the quotient manifold can be defined by means of matrix representatives. Below, we show the development of various geometric objects that are are required to optimize a smooth cost function on the quotient manifold with a first-order algorithm (including the stochastic gradient algorithm). Most of these notions follow directly from \citep{absil08a}.

A fundamental requirement is the characterization of the linearization of the Grassmann manifold, which is the called its \emph{tangent space}. Since the Grassmann manifold is the quotient space of the Stiefel manifold, shown in (\ref{eq:quotient_structure}), its tangent space has matrix representation in terms of the tangent space of the larger Stiefel manifold $ \Stiefel{r}{m}$. Endowing the Grassmann manifold with a Riemannian submersion structure \citep{absil08a}, the tangent space of $ \Stiefel{r}{m}$ at $\mat U$ has the characterization
\begin{equation}\label{eq:tangent_space_Stiefel}
T_{\mat{U}} \Stiefel{r}{m} := \{ \mat{Z}_{\mat{U}} \in \mathbb{R}^{m \times r}: \mat{U}^\top {\mat Z}_{\mat{U}} + {\mat Z}_{\mat{U}}^\top \mat{U} = {\mat 0} \}.
\end{equation}
The tangent space of $\Grass{r}{m}$ at an element $\mathcal{U}$ identifies with a \emph{subspace} of $T_{\mat{U}} \Stiefel{r}{m}$ (\ref{eq:tangent_space_Stiefel}), and specifically, which has the matrix characterization, i.e.,
\begin{equation}\label{eq:tangent_space_Grassmann}
\text{matrix characterization of }T_{\mathcal{U}} \Grass{r}{m} := \{ \xi_{\mat{U}} \in \mathbb{R}^{m \times r}: \mat{U}^\top \xi_{\mat{U}} = 0 \},
\end{equation}
where $\mat{U}$ is the matrix characterization of $\mathcal{U}$. In (\ref{eq:tangent_space_Grassmann}), the vector $\xi_{\mat{U}}$ is the matrix characterization of the abstract tangent vector $\xi_{\mathcal{U}} \in T_{\mathcal{U}} \Grass{r}{m}$ at $\mathcal{U} \in \Grass{r}{m}$.

A second requirement is the computation of the Riemannian gradient of a cost function, say $f : \Grass{r}{m} \rightarrow \mathbb{R}$. Again exploiting the quotient structure of the Grassmann manifold, the Riemannian gradient ${\grad}_{\mathcal{U}} f$ of $f$ at $\mathcal{U} \in \Grass{r}{m}$ admits the matrix expression
\begin{eqnarray*}
	 {\grad}_{\mathcal{U}} f = \Grad_{\mat U} f - \mat{U} (\mat{U}^\top \Grad_{\mat U} f),
\end{eqnarray*}
where $\Grad_{\mat U} f$ is the (Euclidean) gradient of $f$ in the matrix space $\mathbb{R}^{m \times r}$ at $\mat{U}$.

A third requirement is the notion of a straight line along a tangential direction on the Grassmann manifold. This quantity is captured with the \emph{exponential mapping} operation on the Grassmann manifold. Given a tangential direction $\xi_{\mathcal{U}} \in T_{\mathcal{U}} \Grass{r}{m}$ that has the matrix expression $\xi_{\mat U}$ belonging to the subspace (\ref{eq:tangent_space_Grassmann}), the exponential mapping along $\xi_{\mat U}$ has the expression \cite[Section 5.4]{absil08a}
\begin{equation}
\label{Eq:exponential_map}
\Exp_{\mathcal{U}}(\xi_{\mathcal{U}}) :=  \mat{U} \mat{V} \cos(\mat \Sigma)\mat{V}^\top +   \mat{W} \sin(\mat \Sigma) \mat{V}^\top,
\end{equation} 
where $\mat{W \Sigma V}^\top$ is the rank-$r$ singular value decomposition of $\xi_{\mat{U}}$. The $\cos(\cdot)$ and $\sin(\cdot)$ operations are on the diagonal entries.

Finally, a fourth requirement is the notion of the \emph{logarithm} map of an element $\widetilde{\mathcal{U}}$ at $\mathcal{U}$ on the Grassmann manifold. The logarithm map operation maps $\widetilde{\mathcal{U}}$ onto a tangent vector at $\mathcal{U}$, i.e., if $\widetilde{\mathcal{U}}$ and $\mathcal{U}$ have matrix operations $\widetilde{\mat{U}}$ and $\mat{U}$, respectively, then the logarithm map finds a vector in (\ref{eq:tangent_space_Grassmann}) at $\mat{U}$. The closed-form expression of the logarithm map $ {\rm Log}_{\scriptsize \mathcal{U}} (\widetilde{\mathcal{U}})$, i.e.,
\begin{equation}\label{Eq:logarithm_map}
\begin{array}{lll}
 {\rm Log}_{\scriptsize \mathcal{U}(0)} (\mathcal{U}(t))  \ = \ \mat{P} \arctan(\mat S) \mat{Q}^\top,
\end{array}
\end{equation}	
where $\mat{PS} \mat{Q}^\top$ is the rank-$r$ singular value decomposition of $(\widetilde{\mat{U}} - \mat{U} \mat{U}^\top  \widetilde{\mat{U}})\allowbreak(\mat{U}^\top \widetilde{\mat{U}})^{-1}$.

\section{Motivation} \label{sec:motivation}
We look at a decentralized learning of the subspace learning problem of the form
\begin{equation}\label{eq:general_formulation}
\begin{array}{lll}
\displaystyle \min\limits_{\mathcal{U} \in \Grass{r}{m}}  & \displaystyle \sum\limits_{i = 1}^{N} f_i(\mathcal{U}),
\end{array}
\end{equation}
where $\Grass{r}{m}$ is the Grassmann manifold. We assume that the functions $f_i: \mathbb{R}^{m\times r} \rightarrow \mathbb{R}$ for all $i=\{1,\ldots,N \}$ are smooth. In this section, we formulate two popular class of problems as subspace learning problems of the form (\ref{eq:general_formulation}) on the Grassmann manifold. The decentralization learning setting for (\ref{eq:general_formulation}) is considered in Section \ref{sec:decentralized_problem_setup}.

\subsection{Low-rank matrix completion as subspace learning}\label{sec:matrix_completion}
The problem of low-rank matrix completion amounts to completing a matrix from a small number of entries by assuming a low-rank model for the matrix. The rank constrained matrix completion problem can be formulated as
\begin{equation}\label{eq:formulation_mc}
\begin{array}{llll}
	\min\limits_{\mat{Y}\in\mathbb{R}^{m \times n}}
		 \ \ \displaystyle\frac{1}{2}	\|\mathcal{P}_{\Omega}(\mat{Y}) - \mathcal{P}_{\Omega}(\mat{Y}^\star)\|_F^2  +\ \   \lambda \| \mat{Y}  - \mathcal{P}_{\Omega}(\mat Y)\|_F^2 \\
	 \text{subject to}  \ \  \rank(\mat{Y})=r,
\end{array}
\end{equation}
where $\| \cdot\|_F$ is the Frobenius norm, $\lambda$ is the regularization parameter \citep{boumal15a,boumal11a}, and ${\mat{Y }^ \star}\in\mathbb{R}^{n\times m}$ is a matrix whose entries are known for indices if they belong to the subset $(i,j)\in\Omega$ and $\Omega$ is a subset of the complete set of indices $\{(i,j):i\in\{1,...,m\}\text{ and }j\in\{1,...,n\}\}$. The operator $[\mathcal{P}_{\Omega}(\mat{Y})]_{ij}=\mat{Y}_{ij}$ if $(i,j) \in \Omega$ and $[\mathcal{P}_{\Omega}(\mat{Y})]_{ij}=0$ otherwise is called the orthogonal sampling operator and is a mathematically convenient way to represent the subset of known entries. The rank constraint parameter $r$ is usually set to a low value, e.g., $r \ll (m, n)$. The particular regularization term $\| \mat{Y}  - \mathcal{P}_{\Omega}(\mat Y)\|_F^2$ in (\ref{eq:formulation_mc}) is popularly motivated in \citep{dai11a,boumal15a,boumal11a}, and it specifically penalizes the large predictions. An alternative to the regularization term in (\ref{eq:formulation_mc}) is $\| \mat{Y}  \|_F^2$.

A way to handle the rank constraint in (\ref{eq:formulation_mc}) is by using the parameterization $\mat{Y} = \mat{U} \mat{W}^\top$, where $\mat{U} \in \Stiefel{r}{m}$ and $\mat{W} \in \mathbb{R}^{n \times r}$ \citep{boumal15a,boumal11a,mishra14a}. The problem (\ref{eq:formulation_mc}) reads 
\begin{equation}\label{eq:factorized_formulation_mc}
\begin{array}{ll}
\min\limits_{\mat{U}\in \Stiefel{r}{m}}  \min\limits_{\mat{W}\in\mathbb{R}^{n \times r }} \displaystyle \frac{1}{2} \|\mathcal{P}_{\Omega}(\mat{UW}^\top) - \mathcal{P}_{\Omega}(\mat{Y^\star})\|_F^2 \  \ + \ \lambda \| \mat{UW}^\top  - \mathcal{P}_{\Omega}(\mat{UW}^\top)\|_F^2.
\end{array}
\end{equation}
The \emph{inner} least-squares problem in (\ref{eq:factorized_formulation_mc}) admits a closed-form solution. Consequently, it is straightforward to verify that the outer problem in $\mat U$ only depends on the \emph{column space} of $\mat U$, and therefore, is on the Grassmann manifold $\Grass{r}{m}$ and not on $\Stiefel{r}{m}$ \citep{dai12a,boumal15a,boumal11a}. Solving the inner problem in closed form, the problem at hand is 
\begin{equation}\label{eq:grassmann_formulation_mc}
\begin{array}{lll}
\min\limits_{\mathcal{U}\in \Grass{r}{m}} \displaystyle \frac{1}{2} \|\mathcal{P}_{\Omega}(\mat{UW}_{\mat U}^\top) - \mathcal{P}_{\Omega}(\mat{Y^\star})\|_F^2  +\lambda \ \| \mat{UW}_{\mat U}^\top  - \mathcal{P}_{\Omega}(\mat{UW}_{\mat U}^\top)\|_F^2,
\end{array}
\end{equation}
where $\mat{W}_{\mat U}$ is the unique solution to the inner optimization problem in (\ref{eq:factorized_formulation_mc}) and $\mathcal{U}$ is the column space of $\mat U$ \citep{dai12a}. It should be noted that (\ref{eq:grassmann_formulation_mc}) is a problem on the Grassmann manifold $\Grass{r}{m}$, but computationally handled with matrices $\mat U$ in $\Stiefel{r}{m}$.

Consider the case when $\mat{Y}^\star = [\mat{Y}_1^\star, \mat{Y}_2^\star,\ldots, \mat{Y}_N^\star]$ is partitioned along the columns such that the size of $\mat{Y}_i^\star$ is $m\times n_i$ with $\sum n_i = n$ for $i=\{1,2,\ldots, N\}$. $\Omega_i$ is the {\it local} set of indices for each of the partitions. An equivalent reformulation of (\ref{eq:grassmann_formulation_mc}) is the finite sum problem
\begin{equation}\label{eq:semidecentralized_grassmann_formulation_mc}
\begin{array}{lll}
 \min\limits_{\mathcal{U}\in\Grass{r}{m}}  \ \ \displaystyle \sum\limits_{i=1}^{N} 
  f_i(\mathcal U),% \coloneqq  \frac{1}{2} \|\mathcal{P}_{\Omega _i}({{\mat{UW}}}_{i{\mat U}}^\top) - \mathcal{P}_{\Omega _i}(\mat{Y}^\star _i)\|_F^2 \\
 %\qquad     + \lambda \  \| \mat{UW}_{i\mat U}^\top  - \mathcal{P}_{\Omega}(\mat{UW}_{i\mat U}^\top)\|_F^2,
\end{array}
\end{equation}
where $f_i(\mathcal U) \coloneqq  0.5\|\mathcal{P}_{\Omega _i}({{\mat{UW}}}_{i{\mat U}}^\top) - \mathcal{P}_{\Omega _i}(\mat{Y}^\star _i)\|_F^2  + \lambda  \| \mat{UW}_{i\mat U}^\top  - \mathcal{P}_{\Omega}(\mat{UW}_{i\mat U}^\top)\|_F^2$ and ${{\mat{W}}}_{i{\mat U}}$ is the least-squares solution to $\argmin_{\mat{W}_i \in\mathbb{R}^{n_i \times r }}  \|\mathcal{P}_{\Omega _i}(\mat{UW}_{i }^\top) - \mathcal{P}_{\Omega _i}(\mat{Y}_i^\star)\|_F^2 + \lambda \| \mat{UW}_{i }^\top  - \mathcal{P}_{{\Omega}_i}(\mat{UW}_{i }^\top)\|_F^2 $ for each of the data partitions. The problem (\ref{eq:semidecentralized_grassmann_formulation_mc}) is of type (\ref{eq:general_formulation}).

\subsection{Low-dimensional multitask feature learning as subspace learning}\label{sec:multitask}
We next transform an important problem in the multitask learning setting~\citep{Caruana97,baxter97,Evgeniou05} as a subspace learning problem on the Grassmann manifold. The paradigm of multitask learning advocates joint learning of {\it related} learning problems. A common notion of task-relatedness among different tasks (problems) is as follows: {\it tasks share a latent low-dimensional feature representation}~\citep{ando05a,argyriou08a,Zhang08,mjaw11a,Kang11}. We propose to learn this shared feature subspace. We first introduce a few notations related to multitask setting. 

Let $T$ be the number of given tasks, with each task $t$ having $d_t$ training examples. Let $(\mat{X}_t, y_t)$ be the training instances and corresponding labels for task $t=1,\ldots,T$, where $\mat{X}_t\in\mathbb{R}^{d_t\times m}$ and $y_t\in\mathbb{R}^{d_t}$. Argyriou et al. \citep{argyriou08a} proposed the following formulation to learn a shared latent feature subspace: 
\begin{equation}\label{eq:formulation_mtl}
\begin{array}{llll}
	 \min\limits_{\mat{O}\in \mathbb{R}^{m\times m}, {w}_t\in \mathbb{R}^{m}} &\displaystyle \frac{1}{2} \sum_{t}\|\mat{X}_t \mat{O}w_t  -  {y}_t\|_{F}^2 
 + \  \lambda \|\mat{W}^\top\|^2_{2,1}.
\end{array}
\end{equation}
Here, $\lambda$ is the regularization parameter, $\mat{O}$ is an orthogonal matrix of size $m \times m$ that is shared among $T$ tasks, $w_t$ is the weight vector (also know as task parameter) for task $t$, and $\mat{W}:=[w_1, w_2, \ldots, w_T]^\top$.
%Multitask feature covariance learning amounts to learning common features across multiple tasks \citep{argyriou08a}. Assuming that there are $T$ tasks, the multitask learning problem has the formulation
%\begin{equation}\label{eq:formulation_mtl}
%\begin{array}{llll}
%	 \min\limits_{\mat{O}\in \mathbb{R}^{m\times m}, {a}_t\in \mathbb{R}^{m}} & \displaystyle \frac{1}{2} \sum_{t=1}^{T} \|\mat{X}_t \mat{O}a_t  -  {y}_t\|_{F}^2 
%	  + \lambda \|\mat{A}\|^2_{2,1},
%\end{array}
%\end{equation}
%where $(\mat{X}_t, y_t)$ for $t={1,\ldots,T}$ are training instances and corresponding labels with appropriate dimensions and $\lambda$ is the regularization parameter. 
%For example, $\mat{X}_t \in \mathbb{R}^{d_t \times m}$ is the collection of $d_t$ training instances and $y_t \in \mathbb{R}^{d_t}$ is corresponding set of labels. Here $\mat{O}$ is an orthogonal matrix of size $m \times m$ that is shared among $T$ tasks, $a_t$ is the weight vector for task $t$, and $\mat{A}:=[a_1, a_2, \ldots, a_T]$. 
The term $\|\mat{W^\top}\|_{2,1} \coloneqq \sum\limits_j (\sum\limits_i \mat{W}_{ij}^2)^{1/2}$ is the $(2,1)$ norm over the matrix $\mat{W^\top}$. It enforces the {\it group sparse} structure~\citep{Yuan06} across the columns of $\mat{W}$. The sparsity across columns in $\mat{W}$ ensures that we learn a low-dimensional latent feature representation for the tasks. The basis vectors of this low-dimensional latent subspace are the columns of $\mat{O}$ corresponding to non-zeros columns of $\mat{W}$. Hence, solving (\ref{eq:formulation_mtl}) leads to a full rank $m\times m$ latent feature space $\mat{O}$ and performs feature selection (via sparse regularization) in this latent space. This is computationally expensive especially in large-scale applications desiring a low ($r$) dimensional latent feature representation where $r\ll m$. In addition, the sparsity inducing $1$-norm is non-smooth which poses additional optimization challenges. 
%Note that the tasks are coupled by the $(2,1)$ norm over $\mat{A}$. Solving 
%It promotes \emph{sparsity of rows} in the matrix $\mat{A}$, i.e., we identify a few rows of $\mat A$ which are relevant, and equivalently, only a few latent feature representation for all the tasks. 
%It should be noted that both $\mat{O}$ and $\mat{A}$ are learned in (\ref{eq:formulation_mtl}). 
%Since only a few rows 

%An alternative to $\|\mat{A}\|_{2,1}$ is obtained by imposing a rank constraint explicitly on $\mat{A}$ in (\ref{eq:formulation_mtl}) leading to the formulation 
%\begin{equation}\label{eq:rank_constraint_formulation_mtl}
%\begin{array}{llll}
%	 \min\limits_{a_t\in \mathbb{R}^{m}} & \displaystyle \frac{1}{2} \sum_{t=1}^{T} \|\mat{X}_t a_t  -  y_t\|_{F}^2 + \lambda \| \mat{A}\|_F^2\\
%	 \subject & \rank(\mat{A})=r.
%\end{array}
%\end{equation}
%It should be noted that the role of $\mat{O}$ in (\ref{eq:formulation_mtl}) is replaced by the row space of $\mat{A}$ in (\ref{eq:rank_constraint_formulation_mtl}). A similar formulation is proposed by \citet{ando05a}.

We instead learn only the basis vectors of the low-dimensional latent subspace, by restricting the dimension of the subspace~\citep{ando05a,Lapin14}. The proposed $r$-dimensional multitask feature learning problem is
%Using the fixed-rank parameterization $\mat{A} = \mat{UW}^\top$, discussed in Section \ref{sec:matrix_completion}, the problem (\ref{eq:rank_constraint_formulation_mtl}) is parameterized as
\begin{equation}\label{eq:factorized_formulation_mtl}
\begin{array}{llll}
	 \min\limits_{\mat{U}\in \Stiefel{r}{m}} \displaystyle \sum_{t} \min\limits_{w_t\in\mathbb{R}^{r }}& \displaystyle \frac{1}{2}  \|\mat{X}_t \mat{U}w_t  -  y_t\|_{F}^2   + \   \lambda \| w_t \|_2^2,
\end{array}
\end{equation}
where $\mat{U}$ is an $m\times r$ matrix in $\Stiefel{r}{m}$ representing the low-dimensional latent subspace. 
% where $\mat{W}= [w_1, w_2, \ldots, w_T]^\top$. 
Similar to the earlier matrix completion case, the {inner} least-squares optimization problem in (\ref{eq:factorized_formulation_mtl}) is solved in closed form by exploiting the least-squares structure. It is readily verified that the outer problem (\ref{eq:factorized_formulation_mtl}) is on~$\mathcal U$, i.e., the search space is the Grassmann manifold. To this end, the problem is 
\begin{equation}\label{eq:grassmann_formulation_mtl_full}
\begin{array}{llll}
	 \min\limits_{\mathcal{U}\in \Grass{r}{m}} \displaystyle \sum_{t} \displaystyle \frac{1}{2}  \|\mat{X}_t \mat{U}w_{t \mat{U}}  -  y_t\|_F^2,
\end{array}
\end{equation}
where $w_{t{\mat U}}$ is the least-squares solution to $\argmin_{w_t \in\mathbb{R}^{r }}  \|\mat{XU}w_t - y_t \|_F^2 + \lambda \| w_t\|_2^2$. More generally, we distribute the $T$ tasks in (\ref{eq:grassmann_formulation_mtl_full}) into $N$ groups such that $\sum n_i = T$. This leads to the formulation
\begin{equation}\label{eq:grassmann_formulation_mtl}
\begin{array}{llll}
	 \min\limits_{\mathcal{U}\in \Grass{r}{m}} \displaystyle \sum_{i=1}^{N} \left \{ f_i(\mathcal{U}) \coloneqq {\sum_{t \in \mathcal{T}_i}   \displaystyle \frac{1}{2}  \|\mat{X}_t \mat{U}w_{t \mat{U}}  -  y_t\|_F^2} \right \},
\end{array}
\end{equation}
where $\mathcal{T}_i$ is the set of the tasks in group $i$. The problem (\ref{eq:grassmann_formulation_mtl}) is also a particular case of (\ref{eq:general_formulation}).

\section{Decentralized subspace learning with gossip}\label{sec:decentralized_problem_setup}
We exploit the finite sum (sum of $N$ sub cost functions) structure of the problem (\ref{eq:general_formulation}) by distributing the tasks among $N$ agents, which perform certain computations, e.g., computation of the functions $f_i$ given $\mathcal U$, independently. Although the computational workload gets distributed among the agents, all agents require the knowledge of the common $\mathcal U$, which is an obstacle in decentralized learning. To circumvent this issue, instead of one shared subspace $\mathcal U$ for all agents, each agent $i$ stores a local subspace copy $\mathcal{U}_i$, which it then updates based on information from its \emph{neighbors}. For minimizing the communication overhead between agents, we additionally put the constraint that at any time slot only two agents communicate, i.e, each agent has exactly only one neighbor. This is the basis of the standard gossip framework \citep{boyd06a}. A similar architecture is also exploited in \citep{bonnabel13a} for decentralized covariance matrix estimation. It should be noted that although we focus on this agent network, our cost formulation can be extended to any arbitrary network of agents.

Following \citep{bonnabel13a}, the agents are numbered according to their proximity, e.g., for $ i \leqslant N -1$, agents $i$ and $i + 1$ are neighbors. Equivalently, agents $1$ and $2$ are neighbors and can communicate. Similarly, agents $2$ and $3$ communicate, and so on. This communication between the agents allows to reach a \emph{consensus} on the subspaces $\mathcal{U}_i$. Our proposed approach to handle the finite sum problem (\ref{eq:general_formulation}) in a decentralized setting is to solve the problem
\begin{equation}\label{eq:decentralized_grassmann_formulation}
\begin{array}{lll}
 \min\limits_{\mathcal{U}_1 , \ldots, \mathcal{U}_N \in\Grass{r}{m}}   \displaystyle \sum\limits_{i=1}^{N} \underbrace{f_i(\mathcal{U}_i)}_{{\rm task\ handled\ by\ agent\ } i}  + \displaystyle\frac{\rho}{2} \underbrace{( d_1^2 (\mathcal{U}_1, \mathcal{U}_2)+ \ldots+d_{N-1} ^2 (\mathcal{U}_{N-1}, \mathcal{U}_{N}) )}_{\rm consensus\ among\ agents},
\end{array}
\end{equation}
where $d_i$ in (\ref{eq:decentralized_grassmann_formulation}) is specifically chosen as the {\it Riemannian distance} between the subspaces $\mathcal{U}_i$ and $\mathcal{U}_{i + 1}$ for $i\leqslant N-1$ and $\rho \geqslant 0$ is a parameter that trades off individual (per agent) task minimization with consensus. 
 
 %The distance measure $d_i$ in (\ref{eq:decentralized_grassmann_formulation}) is specifically chosen as the {\it Riemannian distance} between the subspaces $\mathcal{U}_i$ and $\mathcal{U}_{i + 1}$.

For a large $\rho$, the consensus term in (\ref{eq:decentralized_grassmann_formulation}) dominates, minimizing which allows the agents to arrive at consensus, i.e., their subspaces converge. For $\rho = 0$, the optimization problem (\ref{eq:decentralized_grassmann_formulation}) solves $N$ independent tasks and there is no consensus among the agents. For a sufficiently large $\rho$, the problem (\ref{eq:decentralized_grassmann_formulation}) achieves the goal of approximate task solving along with approximate consensus. It should be noted that the consensus term in (\ref{eq:decentralized_grassmann_formulation}) has only $N-1$ pairwise distances. For example, \citep{bonnabel13a} uses this consensus term structure for covariance matrix estimation. It allows to parallelize subspace learning, as discussed later in Section~\ref{sec:parallel_variant}. \changeBMM{Additionally, the standard gossip formulation allows to show the benefit of the trade-off weight $\rho$ in practical problems.} 

It should be noted that although we focus on a particular agent-agent network, our cost formulation can be extended to any arbitrary network of agents. For other complex (and communication heavy) agent-agent networks, the consensus part of (\ref{eq:decentralized_grassmann_formulation}) has additional terms.

%In the Euclidean context, i.e., when the search space is the Euclidean space, a natural way of handling decentralized problem formulations is through the use of the alternating direction method of multipliers (ADMM), e.g., \citep{lin15a} exploits ADMM for the problem of low-rank matrix completion. The ADMM approach, however, has no proper counterpart in the manifold optimization framework. This necessitates the use of the simpler formulation (\ref{eq:decentralized_grassmann_formulation}), where the cost and consensus terms are together minimized.

\section{The Riemannian gossip algorithm for (\ref{eq:decentralized_grassmann_formulation})}\label{sec:gossip}
In this section, we focus on proposing a stochastic algorithm for (\ref{eq:decentralized_grassmann_formulation}) by appropriately sampling the terms in the cost function of (\ref{eq:decentralized_grassmann_formulation}). This leads to simpler updates of the agent specific subspaces. Additionally, it allows to exploit parallelization of updates. To this end, we exploit the stochastic gradient algorithm framework on Riemannian manifolds \citep{bonnabel13a,sato17a,zhang16a}.

As a first step, we reformulate the problem (\ref{eq:decentralized_grassmann_formulation}) as a single sum problem, i.e.,
\begin{equation}\label{eq:sampling_decentralized_grassmann_formulation}
\begin{array}{lll}
 \min\limits_{\mathcal{U}_1 , \ldots, \mathcal{U}_N \in\Grass{r}{m}}  & \displaystyle \sum\limits_{i = 1}^{N-1} g_i(\mathcal{U}_i, \mathcal{U}_{i+1}) ,
\end{array}
\end{equation}
where $g_i(\mathcal{U}_i, \mathcal{U}_{i+1}) := \alpha_i f_k(\mathcal{U}_i) + \alpha_{i+1} f_{i+1}(\mathcal{U}_{i+1}) + 0.5\rho d_k^2 (\mathcal{U}_{i}, \mathcal{U}_{i+1})$. Here, $\alpha_i$ is a scalar that ensures that the cost functions of (\ref{eq:sampling_decentralized_grassmann_formulation}) and (\ref{eq:decentralized_grassmann_formulation}) remain the same with the reformulation, i.e., $\sum g_i = f_1 +\ldots+ f_N + 0.5\rho ( d_1^2 (\mathcal{U}_1, \mathcal{U}_2) +  d_2 ^2(\mathcal{U}_2, \mathcal{U}_3) +\ldots+ d_{N-1} ^2 (\mathcal{U}_{N-1}, \mathcal{U}_{N}) )$. Equivalently, $\alpha_i = 1$ if $i=\{1, N \}$, else $\alpha_i = 0.5$.

\begin{algorithm}[t]
    \caption{Proposed stochastic gossip algorithm for (\ref{eq:sampling_decentralized_grassmann_formulation}).}
    \label{alg:online_algorithm}
    
    \begin{enumerate}
\item At each time slot $k$, pick $g_i $ with $i \leqslant N-1$ randomly with uniform probability. This is equivalent to picking up the agents $i$ and $i +1$.

\item Compute the Riemannian gradients $\grad_{\mathcal{U}_i} g_i$ and $\grad_{{\mathcal U}_{i+1}} g_{i}$.% of the function $g_i$ at $\mathcal{U}_i$ and $\mathcal{U}_{i+1}$, respectively.
\item Given a stepsize $\gamma_k$ (e.g., $\gamma_k \coloneqq a/(1 + b k)$; $a$ and $b$ are constants), update $\mathcal{U}_i$ and $\mathcal{U} _{i + 1}$ as 
\[
\begin{array}{lllll}
({\mathcal{U}_i})_+  = \Exp_{\mathcal{U}_i} (- \gamma_k \grad_{\mathcal{U}_i} g_i   ) \\
({\mathcal{U}_{i + 1}})_+  =  \Exp_{\mathcal{U}_{i+1}}( - \gamma_k  \grad_{\mathcal{U}_{i+1}} g_{i}), \\
\end{array}
\]
where $ ({\mathcal{U}_i})_+$ and $({\mathcal{U}_{i + 1}})_+$ are the updated subspaces and $\Exp_{\mathcal{U}_i}(\xi_{\mathcal{U}_i})$ is the \emph{exponential} mapping that maps the tangent vector $\xi_{\mathcal{U}_i} \in T_{\mathcal{U}_i} \Grass{r}{m}$ onto $\Grass{r}{m}$.
\item Repeat.
\end{enumerate}
\end{algorithm}

At each iteration of the stochastic gradient algorithm, we sample a sub cost function $g_i$ from the cost function in (\ref{eq:sampling_decentralized_grassmann_formulation}) uniformly at random (we stick to this sampling process for simplicity). Based on the chosen sub cost function, the subspaces $\mathcal{U}_i$ and $\mathcal{U}_{i+1}$ are updated by following the negative Riemannian gradient (of the sub cost function $g_i$) with a stepsize. The stepsize sequence over the iterations satisfies the conditions that it is square integrable and its summation is divergent (this is explicitly mentioned in the proof of Proposition \ref{prop:convergence} later). 

The overall algorithm is listed as Algorithm \ref{alg:online_algorithm}, which converges to a critical point of (\ref{eq:sampling_decentralized_grassmann_formulation}) \emph{almost surely} \citep{bonnabel13a}. An outcome of the updates from Algorithm \ref{alg:online_algorithm} is that agents $1$ and $N$ update twice the number of times the rest of agents update.

The matrix characterizations of implementing Algorithm \ref{alg:online_algorithm} are shown in Table \ref{tab:matrix_expressions}. The development of some of the expressions are discussed earlier in Section \ref{sec:grassmann}. The asymptotic convergence analysis of Algorithm \ref{alg:online_algorithm} follows directly from the proposition below. %the analysis in \citep{bonnabel13a}. %In particular, we have the following proposition.

\begin{table}[t]
\caption{Matrix characterizations of ingredients needed to implement Algorithm \ref{alg:online_algorithm}.}\label{tab:matrix_expressions} 
\begin{center}  %\scriptsize %   %\footnotesize %     
\begin{tabular}{p{2cm} p{9.5cm}}
Ingredients & Matrix formulas\\
\toprule
%\hline \\
$d_i^2(\mathcal{U}_i, \mathcal{U}_{i+1})$ &$ 0.5\|\Log_{\mathcal{U}_i}(\mathcal{U}_{i + 1})\|_{F}^2$.\vspace{3pt}\\
%& \\
$\Log_{\mathcal{U}} (\widetilde{\mathcal{U}})$ & $\mat{P} {\rm arctan}(\mat{S})\mat{Q}^\top$, where $\mat{PS Q}^\top$ is the rank-$r$ singular value decomposition of $( \widetilde{\mat{U}} -   \mat{U}  (\mat{U}^\top \widetilde{\mat{U}}))  (\mat{U}^\top \widetilde{\mat{U}})^{-1} $. Here, $\mat{U}$ and $\widetilde{\mat{U}}$ are the matrix representations of $\mathcal{U}$ and $\widetilde{\mathcal{U}}$. \vspace{3pt}\\
%& \\
$g_i(\mathcal{U}_i, \mathcal{U}_{i+1}) $ & $\alpha_i f_i + 0.5\rho d_i^2(\mathcal{U}_i, \mathcal{U}_{i +1})$ \vspace{3pt}.\\
%& \\
$\grad_{\mathcal{U}_i} g_i$ &

$\alpha_i \grad_{\mathcal{U}_i} f_i + \rho \grad_{\mathcal{U}_i}d_i$\vspace{3pt}.\\
%& \\
$\grad_{\mathcal{U}_i} f_i $ &  $\Grad_{\mathcal{U}_i}f_i - \mat{U}_i (\mat{U}_i ^\top \Grad_{\mathcal{U}_i}f_i)$, where for the matrix completion cost (\ref{eq:semidecentralized_grassmann_formulation_mc}), $\Grad_{\mathcal{U}_i} f_i $ is
$\Grad_{\mathcal{U}_i} f_i = (\mathcal{P}_{\Omega _i}({{\mat{U}_i\mat{W}}}_{i{\mat{U}_i}}^\top) - \mathcal{P}_{\Omega _i}(\mat{Y}^\star _i) )\mat{W}_{i{\mat{U}_i}}$, and for the multitask feature learning cost (\ref{eq:grassmann_formulation_mtl}) $\Grad_{\mathcal{U}_i} f_i $ is $\Grad_{\mathcal{U}_i} f_i = \sum_{t \in \mathcal{T}_i} \mat{X}_t^\top(\mat{X}_t \mat{U}w_{t{\mat{U}_i}} -  y_t)w^\top_{t{\mat{U}_i}}$. Here, $\mat{W}_{i{\mat{U}_i}}$ and $w_{t\mat{U}_i}$ are the solutions of the inner least-squares problems for the respective problems. \vspace{3pt}
 \\
 %& \\
$\grad_{\mathcal{U}_i} d_i$ & $-\Log_{\mathcal{U}_i}(\mathcal{U}_{i+1})$ \citep{bonnabel13a}. \vspace{3pt}  \\
%& \\
$\Exp_{\mathcal{U}_i}(\xi_{\mathcal{U}_i}) $ & $ \mat{U}_{i} \mat{V} \cos(\mat \Sigma)\mat{V}^\top +   \mat{W} \sin(\mat \Sigma) \mat{V}^\top$, where $\mat{W \Sigma V}^\top$ is the rank-$r$ singular value decomposition of $\xi_{\mathcal{U}_i}$. The $\cos(\cdot)$ and $\sin(\cdot)$ operations are on the diagonal entries. \\
\bottomrule
\end{tabular}
\end{center} 
\end{table}

%\subsection{Convergence analysis}\label{sec:convergence_analysis}

\begin{proposition}\label{prop:convergence}
Algorithm~\ref{alg:online_algorithm} converges to a first-order critical point of (\ref{eq:sampling_decentralized_grassmann_formulation}).
\end{proposition}
\begin{proof}
The problem (\ref{eq:sampling_decentralized_grassmann_formulation}) can be modeled as
\begin{equation}\label{eq:proof}
\begin{array}{lll}
\min\limits_{\mathcal{V} \in \mathcal{M}} & \displaystyle \frac{1}{N-1} \sum\limits_{i=1}^{N-1}h_i(\mathcal{V}), 
\end{array}
\end{equation}
where $\mathcal{V}:=(\mathcal{U}_1, \mathcal{U}_2,\ldots, \mathcal{U}_N)$, $\mathcal{M}$ is the Cartesian product of $N$ Grassmann manifolds $\Grass{r}{m}$, i.e., $\mathcal{M}:= {\rm Gr}^N(r,m)$, and $h_i: \mathcal{M} \rightarrow \mathbb{R}: \mathcal{V} \mapsto h_i(\mathcal{V}) = g_i(\mathcal{U}_i, \mathcal{U}_{i+1})$. The updates shown in Algorithm~{\ref{alg:online_algorithm}} precisely correspond to stochastic gradients updates for the problem (\ref{eq:proof}).

It should be noted that $\mathcal{M}$ is {\it compact} and has a Riemannian structure, and consequently, the problem (\ref{eq:proof}) is an empirical risk minimization problem on a compact manifold. The key idea of the proof is that for a compact Riemannian manifold, all continuous functions of the parameter are bounded, e.g., the Riemannian Hessian of $h(\mathcal{V})$ is upper bounded for all $\mathcal{V} \in \mathcal{M}$. We assume that 1) the stepsize sequence satisfies the condition that $\sum \gamma_k = \infty$ and $ \sum (\gamma_k)^2  < \infty$ and 2) at each time slot $k$, the stochastic gradient estimate $\grad_{\mathcal V} h_i$ is an {\it unbiased} estimator of the batch Riemannian gradient $\sum_i \grad_{\mathcal{V}} h_i$. Under those assumptions, Algorithm~\ref{alg:online_algorithm} converges to a first-order critical point of (\ref{eq:sampling_decentralized_grassmann_formulation}). %\noteBM{A rigorous convergence analysis of stochastic gradients on manifolds is presented in \citep{anonymous17a}.}
\end{proof}

\subsection{Computational complexity}
\label{sec:computational_cost_analysis}
For an update of $\mathcal{U}_i$ with the formulas shown in Table \ref{tab:matrix_expressions}, the computational complexity depends on the computation of partial derivatives of the cost functions in (\ref{eq:formulation_mc}) and $(\ref{eq:formulation_mtl})$, e.g., the gradient $\Grad_{\mathcal{U}_i} f_i$ computation of agent $i$. 
\begin{itemize}
\item \textbf{Task-related computations.} In the matrix completion problem (\ref{eq:formulation_mc}), the computational cost of the partial derivatives for agent $i$ is $O(|\Omega_i|r^2 + n_i r ^3 + m r^2)$. In the multitask feature learning problem (\ref{eq:formulation_mtl}), the computational cost is $O(m|\mathcal{T}_i| r^2 + |\mathcal{T}_i| r ^3 + m r^2 + (\sum_{t = \mathcal{T}_i} d_t)m)$, where $\mathcal{T}_i$ is the group of tasks assigned to agent $i$.  

\item \textbf{Consensus-related computations.} Communication between agents $i$ and $i+1$ involves computing $d_i(\mathcal{U}_i,\mathcal{U}_{i+1})$ which costs $O(mr^2 + r^3)$.

\item \textbf{Manifold-related computations.} Computing the exponential and logarithm mappings cost $O(mr^2 + r^3)$. Computation of the Riemannian gradient costs $O(mr^2)$.

\end{itemize}

%\begin{table}[t]
%\caption{Proposed preconditioned gossip algorithm for (\ref{eq:decentralized_grassmann_formulation})}
%\label{tab:precon_online_algorithm} 
%\begin{center} \small
%\begin{tabular}{ |p{13.5cm}| }
%\hline
%\begin{enumerate}
%\item At each time slot $t$, pick an agent $i \leqslant N-1$ randomly with uniform probability and compute the Riemannian gradients $\grad_{x_i} f_i$, $\grad_{x_{i+1}} f_{i + 1}$, $\grad_{x_i} d_i$, and $\grad_{x_{i+1}} d_{i}$ with the matrix representations shown in Table \ref{tab:online_algorithm}.
%\item Given a stepsize $\gamma_t$, update $x_i$ and $x_{i + 1}$ as 
%\[
%\begin{array}{lllll}
%{x_i}_+  = \Exp_{x_i} (- \gamma_t (\alpha_i \grad_{x_i} f_i + \rho    \grad_{x_i} d_i ) ({\mat{W}_{i{\mat U}_i} ^\top \mat{W}_{i{\mat U}_i}} + {\rho \mat{I}})^{-1}) \\
%{x_{i + 1}}_+  =  \Exp_{x_{i+1}}( - \gamma_t (  \alpha_{i+1} \grad_{x_{i+1}} f_{i + 1} + \rho \grad_{x_{i+1}} d_{i}) ({\mat{W}_{{i+1}{\mat U}_{i+1}} ^\top \mat{W}_{{i+1}{\mat U}_{i+1}}} + {\rho \mat{I}})^{-1} ), \\
%\end{array}
%\]
%where $\mat{W}_{i{\mat U}_i}$ is the least-squares solution to the optimization problem $\min_{\mat{W}_i \in \mathbb{R}^{n_i \times r }}  \|\mathcal{P}_{\Omega _i }(\mat{U}_i \mat{W}_i^\top) - \mathcal{P}_{\Omega _i}(\mat{X}_i ^\star)\|_F^2$. $\Exp$ and $\alpha_i$ are defined in Table \ref{tab:online_algorithm}.
%%\item Repeat.
%\end{enumerate}
%  \\
% % \\  
% \hline
%\end{tabular}
%\end{center} 
%\end{table}  

\subsection{Preconditioned variant}\label{sec:preconditioned_variant}
The performance of first order algorithms (including stochastic gradients) often depends on the \emph{condition number} of the Hessian of the cost function (at the minimum). For the matrix completion problem (\ref{eq:formulation_mc}), the issue of {\it ill-conditioning} arises when data $\mat{Y}^\star$ have power law distributed singular values. Additionally, a large value of $\rho$ in (\ref{eq:decentralized_grassmann_formulation}) leads to convergence issues for numerical algorithms. The recent works \citep{ngo12a,mishra14c,boumal15a} exploit the concept of \emph{manifold preconditioning} for the matrix completion problem (\ref{eq:formulation_mc}). In particular, the Riemannian gradients are \emph{scaled} by computationally cheap matrix terms that arise from the second order curvature information of the cost function. This operation on a manifold requires special attention. In particular, the matrix scaling \emph{must} be a positive definite operator on the {\it tangent space} of the manifold \citep{mishra14c,boumal15a}. 

Given the Riemannian gradient, e.g,  $\grad_{\mathcal{U}_i} g_i$ for agent $i$, the proposed preconditioner for (\ref{eq:decentralized_grassmann_formulation}) is
\begin{equation}\label{eq:preconditioner}
\grad_{\mathcal{U}_i} g_i \mapsto (\grad_{\mathcal{U}_i} g_i)  (\underbrace{\mat{W}_{i{\mat U}_i} ^\top \mat{W}_{i{\mat U}_i}}_{\text{from the task term}} \ \ + \ \ \underbrace{\rho \mat{I}}_{\text{from the consensus term}})^{-1},
\end{equation}
where $\mat{I}$ is the $r\times r$ identity matrix. The use of preconditioning (\ref{eq:preconditioner}) costs $O(n_i r^2 + r^3)$, which is computationally cheap to implement. The term $\mat{W}_{i{\mat U}_i} ^\top \mat{W}_{i{\mat U}_i}$ captures a \emph{block diagonal approximation} of the Hessian of the simplified (but related) cost function $ \|{{\mat{U}_i\mat{W}}}_{i{\mat{U}_i}}^\top - \mat{Y}^\star _i \|_F^2$, i.e., an approximation for (\ref{eq:formulation_mc}) and (\ref{eq:formulation_mtl}) \citep{ngo12a,mishra14c,boumal15a}. The term $\rho \mat{I}$ is an approximation of the second order derivative of the square of the Riemannian distance. Finally, it should be noted that ${\mat{W}_{i{\mat U}_i} ^\top \mat{W}_{i{\mat U}_i}}  + {\rho \mat{I}} \succ 0$.

% Equivalently, if $\xi_{x_i}$ belongs to $T_{x_i} \mathcal{M}$, then $\xi_{x_i} (\mat{W}_{i{\mat U}_i} ^\top \mat{W}_{i{\mat U}_i}  + \rho \mat{I})^{-1}$ also belongs to $T_{x_i} \mathcal{M}$. This is readily checked by the fact that the tangent space $T_{x_i} \mathcal{M}$ at $x_i$ on the Grassmann manifold is characterized by the set $\{ \eta_{x_i}:\eta_{x_i} \in \mathbb{R}^{m \times r},\mat{U}_{i}^\top \eta_{x_i} = 0 \}$.

%The proposed preconditioned variant of the stochastic gradient  algorithm for (\ref{eq:decentralized_grassmann_formulation}) is shown in Table \ref{tab:precon_online_algorithm}. It should be noted that preconditioning the gradients does not affect the asymptotic convergence guarantees of the proposed algorithm.
%

\subsection{Parallel variant}\label{sec:parallel_variant}
The particular structure (also known as the red-black ordering structure in domain decomposition methods) of the cost terms in (\ref{eq:sampling_decentralized_grassmann_formulation}), allows for a straightforward parallel update strategy for solving (\ref{eq:sampling_decentralized_grassmann_formulation}). We look at the following separation of the costs, i.e., the problem is 
\begin{equation}\label{eq:parallel_sampling_decentralized_grassmann_formulation}
\begin{array}{lll}
 \min\limits_{\mathcal{U}_1 , \ldots, \mathcal{U}_N \in\Grass{r}{m}}  \ \  \displaystyle  {\underbrace{g_1 + g_3  + \ldots}_{g_{\rm odd}}}
 \ \  + \ \ {\underbrace{g_2 + g_4  + \ldots}_{g_{\rm even}}},
\end{array}
\end{equation}
where the subspace updates corresponding to $g_{\rm odd}$ (and similarly $g_{\rm even}$) are {\it parallelizable}. 

We apply Algorithm~\ref{alg:online_algorithm}, where we pick the sub cost function $g_{\rm odd}$ (or $g_{\rm even}$) with uniform probability. The key idea is that sampling is on $g_{\rm odd}$ and $g_{\rm even}$ and not on the sub cost functions $g_i$ directly. This strategy allows to perform $\lfloor (N-1)/2 \rfloor$ updates in parallel. %Experiments in Section \ref{sec:numerical_comparisons_benefit_of_Grassmann} show the p

\section{Numerical comparisons}\label{sec:comparisons}
Our proposed algorithm (Stochastic Gossip) presented as Algorithm~\ref{alg:online_algorithm} and its preconditioned (Precon Stochastic Gossip) and parallel (Parallel Gossip and Precon Parallel Gossip) variants are compared on various different benchmarks on matrix completion and multitask problems. In many cases, our decentralized gossip algorithms match the generalization performance of competing (tuned) batch algorithms.

Stochastic algorithms with $N$ agents are run for a maximum of $200(N-1)$ iterations. The parallel variants are run for $400N$ iterations. Overall, because of the agent-agent network structure, agents $1$ and $N$ end up performing a maximum of $200$ updates and rest all other agents perform $400$ updates. The stepsize sequence is defined as $\gamma_k = a/(1 + bk)$, where $k$ is the time slot. The constants $a$ and $b$ are set using $5$-fold cross validation on the training data. 

Our implementations are based on the Manopt toolbox \citep{boumal14a}. 
%For the matrix completion problem, additionally, certain operations rely on the mex files supplied with RTRMC \citep{boumal15a}. 
All simulations are performed in Matlab on a $2.7$ GHz Intel Core i$5$ machine with $8$ GB of RAM. The comparisons on the  Netflix and MovieLens-10M datasets are performed on a cluster with larger memory.

%\begin{figure}[t]
%\centering
%\begin{tabular}{cc}
%\noindent \begin{minipage}[b]{0.5\hsize}
%\centering
%\includegraphics[width=\hsize]{figures/Grassmann_vs_Euclidean/cost_vs_iterations_agents.pdf}\\
%{\scriptsize (a) Performance of the agents on solving the local tasks. }
%\end{minipage}
%\begin{minipage}[b]{0.5\hsize}
%\centering
%\includegraphics[width=\hsize]{figures/Grassmann_vs_Euclidean/cost_vs_iterations_consensus.pdf}\\
%{\scriptsize (b) The Grassmann geometry allows to agents to achieve consensus.}
%\end{minipage}
%\end{tabular}
%\caption{Exploiting the Grassmann geometry leads to better optimization. The weight factor $\rho$ is best tuned for both the algorithms. This experiment is on a matrix completion problem instance. Figures best viewed in color. }
%\label{fig:all_benefit}
%\end{figure}

\begin{figure}[t]
\centering
\begin{minipage}[b]{0.48\hsize}
\centering
\includegraphics[width=\hsize]{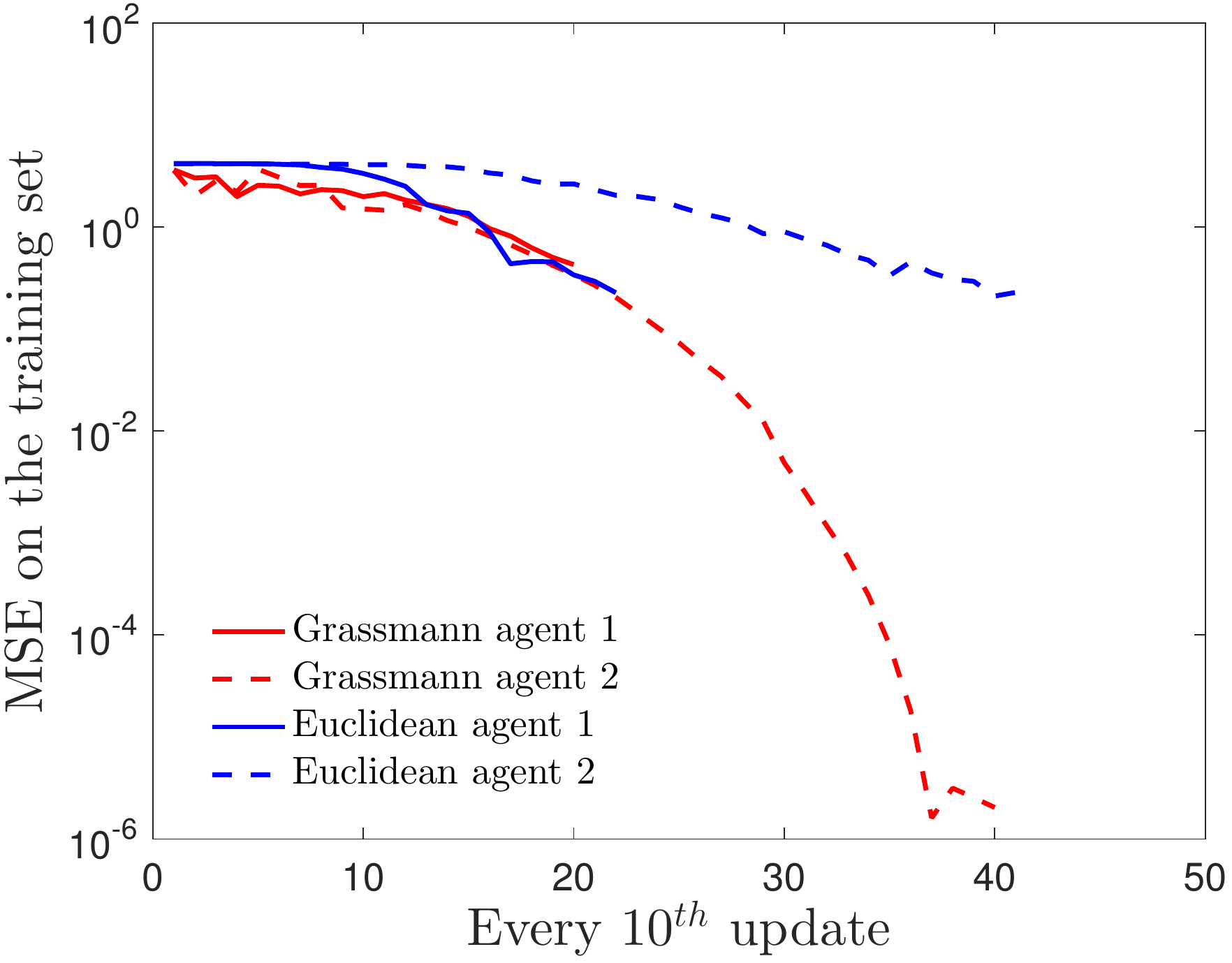}\\
\justify{(a) Performance of the agents on solving the local matrix completion tasks. }
\end{minipage}
\hfill
\begin{minipage}[b]{0.48\hsize}
\centering
\includegraphics[width=\hsize]{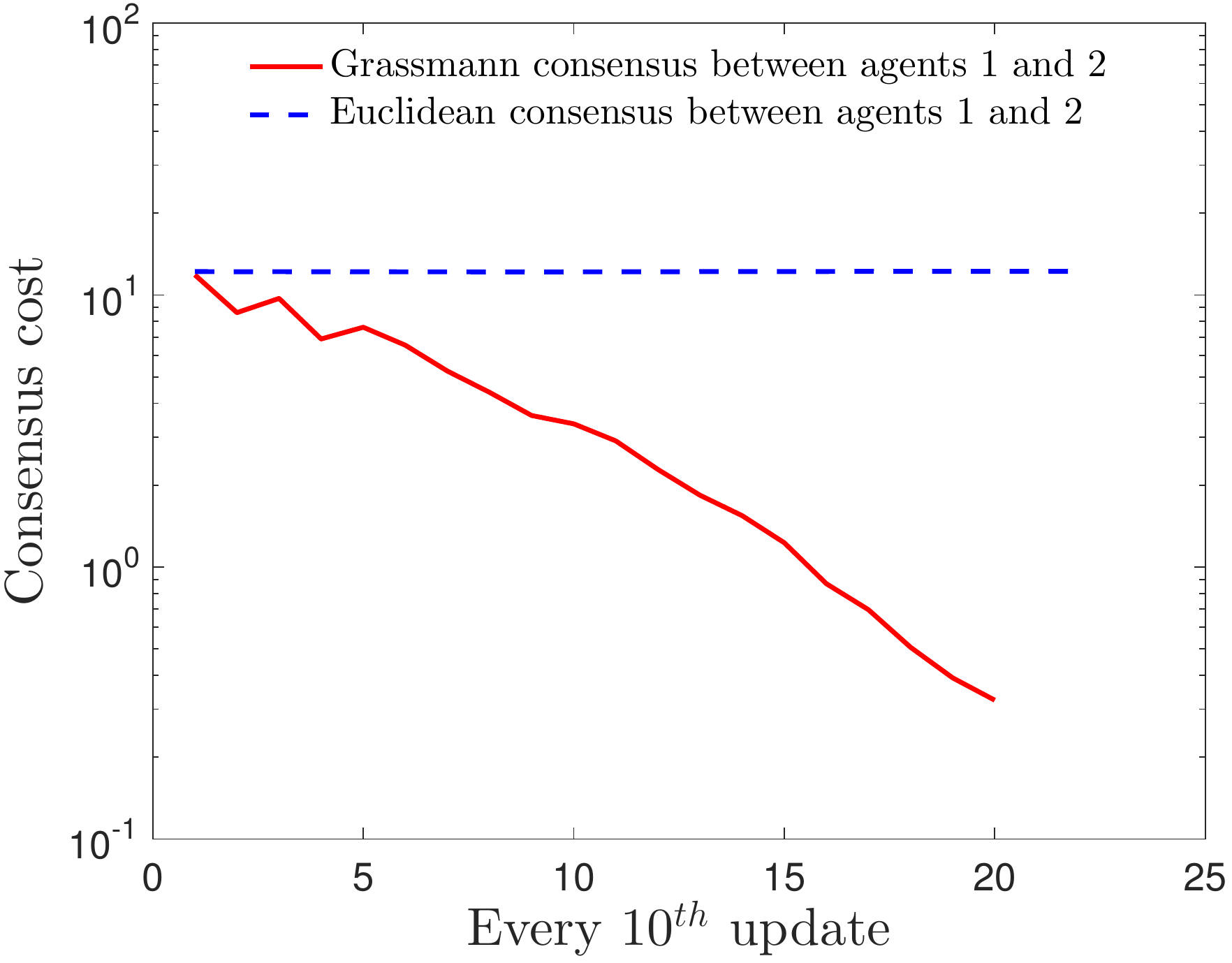}\\
\justify{(b) The Grassmann geometry allows to agents to achieve consensus.}
\end{minipage}
\caption{Exploiting the Grassmann geometry leads to better optimization. The weight factor $\rho$ is best tuned for both the algorithms. This experiment is on a matrix completion problem instance. Figures best viewed in color. }
\label{fig:all_benefit}
\end{figure}

\subsection{Benefit of the Grassmann geometry against the Euclidean geometry} 

In contrast to the proposed formulation (\ref{eq:decentralized_grassmann_formulation}), an alternative is to consider the formulation
\begin{equation}\label{eq:decentralized_euclidean_formulation}
\begin{array}{lll}
 \min\limits_{\mat{U}_1 , \ldots, \mat{U}_N \in \mathbb{R}^{m\times r}}\   \displaystyle \sum\limits_{i} {f_i(\mat{U}_i)} + \displaystyle\frac{\rho}{2} {( \| \mat{U}_1 - \mat{U}_2\|_F^2   +\ldots+ \| \mat{U}_{N-1} - \mat{U}_N\|_F^2)},
\end{array}
\end{equation} 
where the problem is in the {\it Euclidean} space and the consensus among the agents is with respect to the Euclidean distance. Although this alternative choice is appealing for its numerical simplicity, the benefit of exploiting the geometry of the problem is shown in Figure \ref{fig:all_benefit}. We consider a matrix completion problem instance in Figure \ref{fig:all_benefit}, where we apply Stochastic Gossip algorithms with $N=6$ agents. Figure \ref{fig:all_benefit} shows the performance of only two agents for clarity, where agent $1$ performs $200$ updates and agent $2$ performs $400$ updates. This because of the agent-agent network structure as discussed in Section \ref{sec:gossip}.  As shown in Figure \ref{fig:all_benefit}, the algorithm with the Euclidean formulation (\ref{eq:decentralized_euclidean_formulation}) performs poorly due to a very slow rate of convergence. Our approach, on the other hand, exploits the geometry of the problem and obtains a lower mean squared error (MSE).

%Here, we consider a problem instance of size $10\, 000 \times 100\, 000$ of rank $5$ and OS $6$. The parameter $\rho$ is set to $10^{3}$ that achieves completion with consensus of agents. 

\subsection{Matrix completion comparisons} \label{sec:numerical_comparisons_benefit_of_Grassmann}

For each synthetic example considered here, two matrices $\mat{A} \in \mathbb{R}^{m \times r}$ and $\mat{B} \in \mathbb{R}^{n \times r}$ are generated according to a Gaussian distribution with zero mean and unit standard deviation. The matrix product $\mat{AB} ^\top$ gives a random matrix of rank $r$ \citep{cai10a}. A fraction of the entries are randomly removed with uniform probability. Noise (sampled from the Gaussian distribution with mean zero and standard deviation $10^{-6} $) is added to each entry to construct the training set $\Omega$ and $\mat{Y}^\star$. The over-sampling ratio (OS) is the ratio of the number of known entries to the matrix dimension, i.e, ${\rm OS} = |\Omega|/(mr +nr -r^2)$. We also create a test set by randomly picking a small set of entries from $\mat{AB} ^\top$. The matrices $\mat{Y}^\star _i$ are created by distributing the number of $n$ columns of $\mat{Y}^\star$ {\it equally} among the agents. The train and test sets are also partitioned similarly among $N$ agents. All the algorithms are initialized randomly and the regularization parameter $\lambda$ in (\ref{eq:grassmann_formulation_mc}) is set to $\lambda = 0$ for all the below considered cases (except in Case $5$ below, where $\lambda = 0.01$).

\textbf{Case 1: effect of $\rho$.} Here, we consider a problem instance of size $10\, 000 \times 100\, 000$ of rank $5$ and OS $6$. Two scenarios with $\rho = 10^3$ and $\rho = 10^{10}$ are considered. Figures \ref{fig:all}(a)\&(b) show the performance of Stochastic Gossip. Not surprisingly, for $\rho = 10^{10}$, we only see consensus (the distance between agents $1$ and $2$ tends to zero). For $\rho = 10^3$, we observe both a low MSE on the matrix completion problem as well as  consensus among the agents.

\textbf{Case 2: performance of Stochastic Gossip versus Parallel Gossip.} We consider Case 1 with $\rho = 10^3$. Figures \ref{fig:all}(c)\&(d) show the performance of Stochastic Gossip and Parallel Gossip, both of which show a similar behavior on the training set (as well as on the test set, which is not shown here for brevity).

\begin{figure}[!t]
\centering
\begin{minipage}[b]{0.48\hsize}
\centering
\includegraphics[width=\hsize]{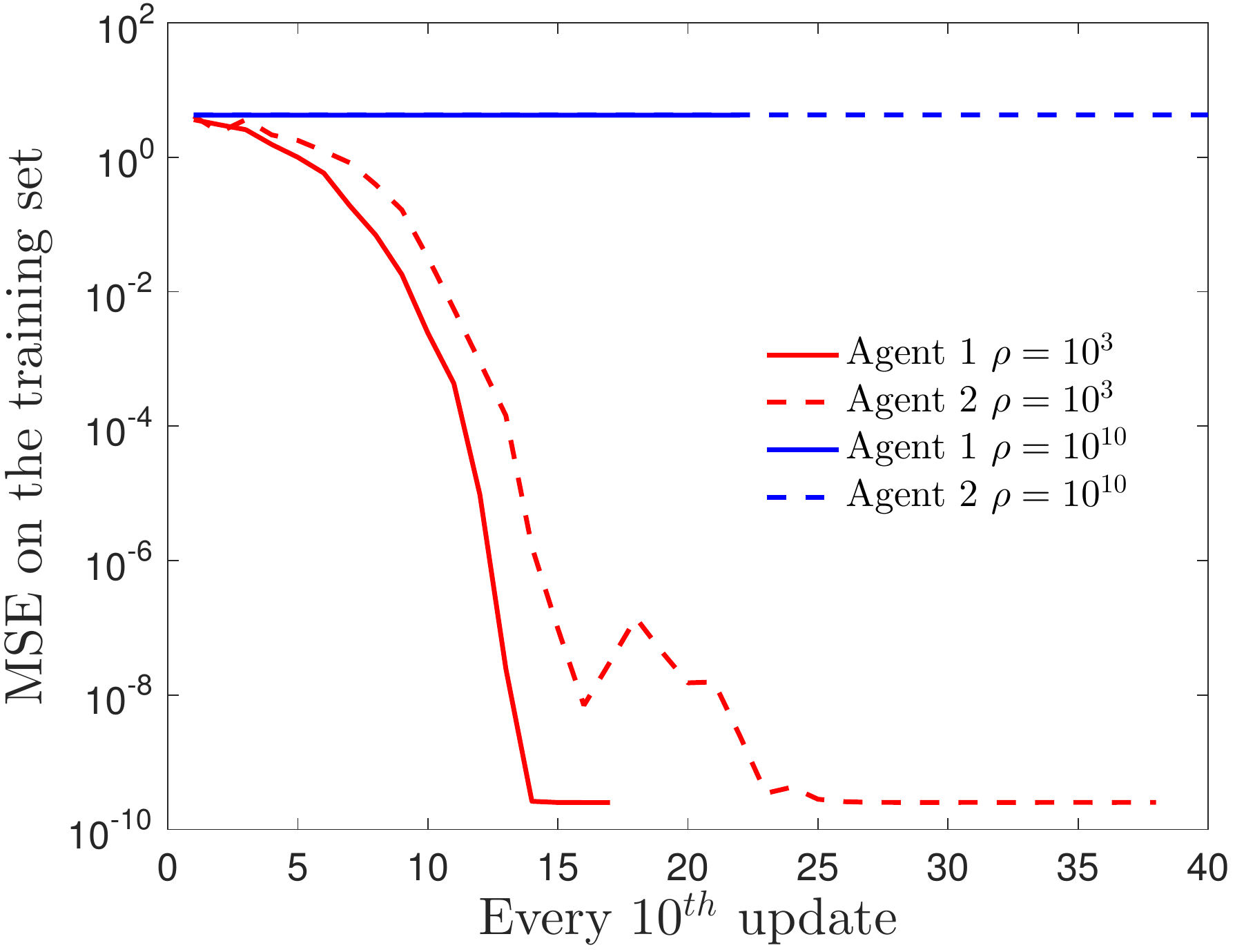}
\justify{(a) Agents successfully learn subspace with a tuned $\rho$ parameter (Case 1).}
\end{minipage}
\hfill
\begin{minipage}[b]{0.48\hsize}
\centering
\includegraphics[width=\hsize]{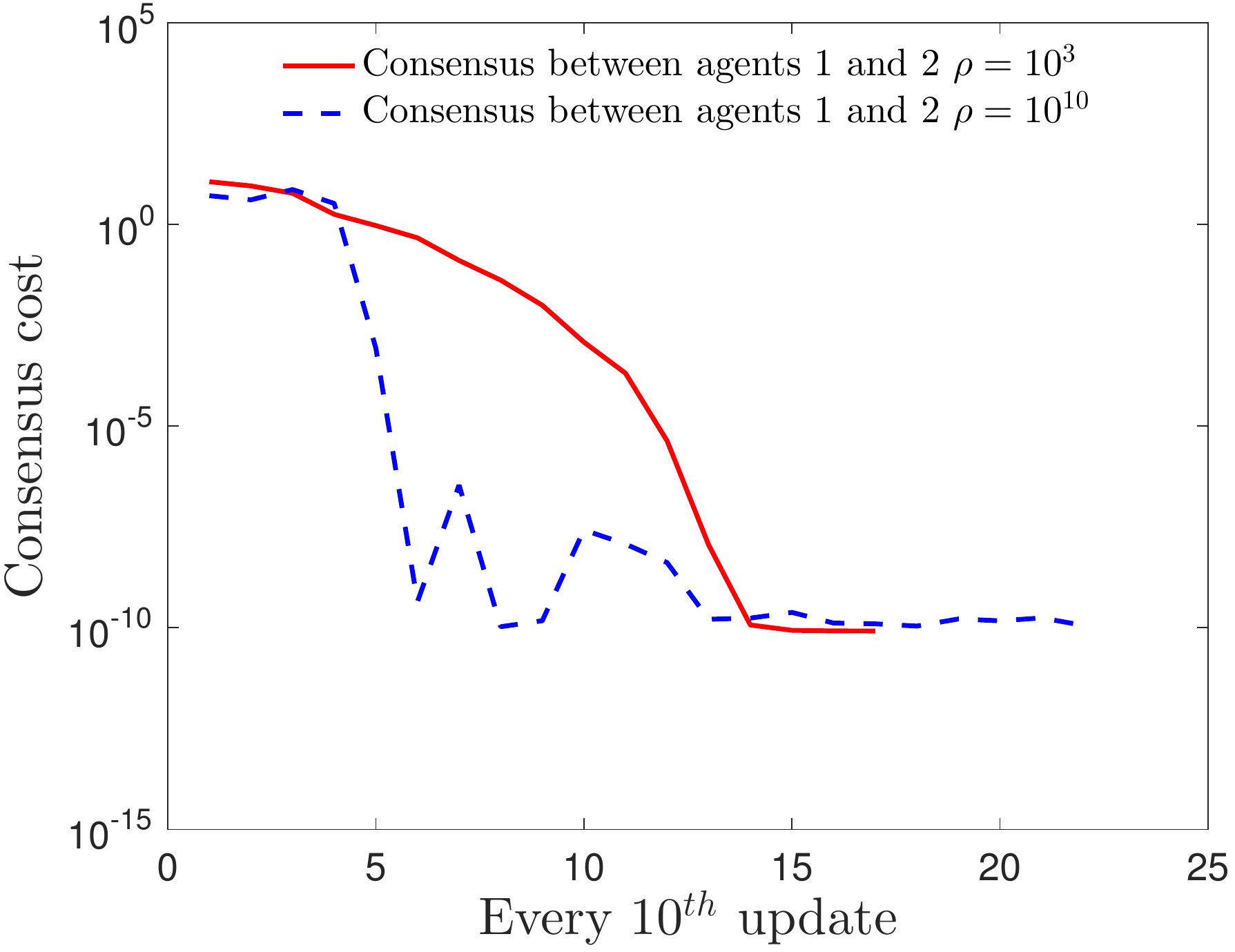}
\justify{(b) Consensus between the agents is obtained with a tuned $\rho$ parameter (Case 1).}
\end{minipage} \\
\begin{minipage}[b]{0.48\hsize}
\centering
\includegraphics[width=\hsize]{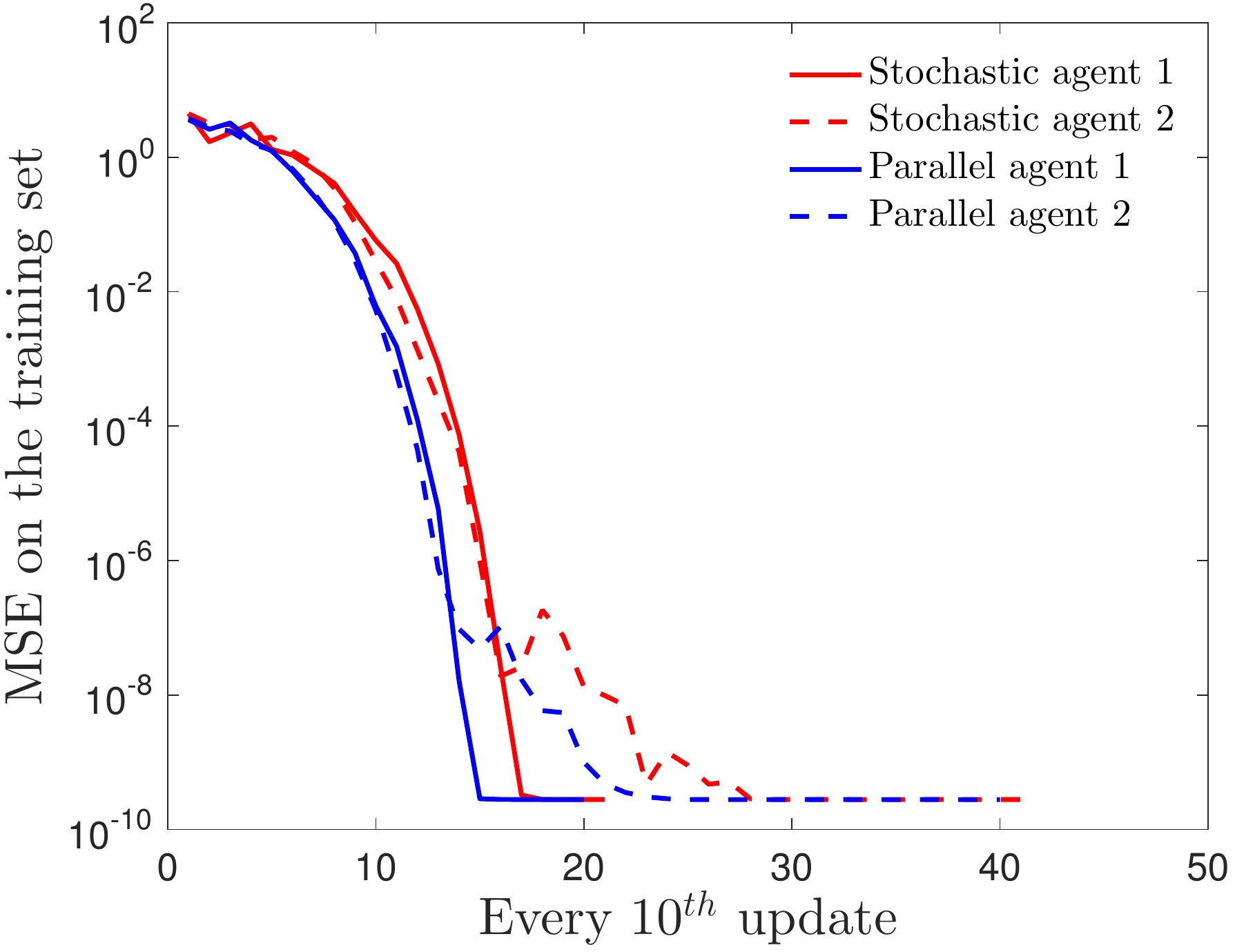}
\justify{(c) Stochastic and parallel variants perform similarly on local subspace learning tasks (Case 2).}
 \end{minipage}
 \hfill
 \begin{minipage}[b]{0.48\hsize}
\centering
\includegraphics[width=\hsize]{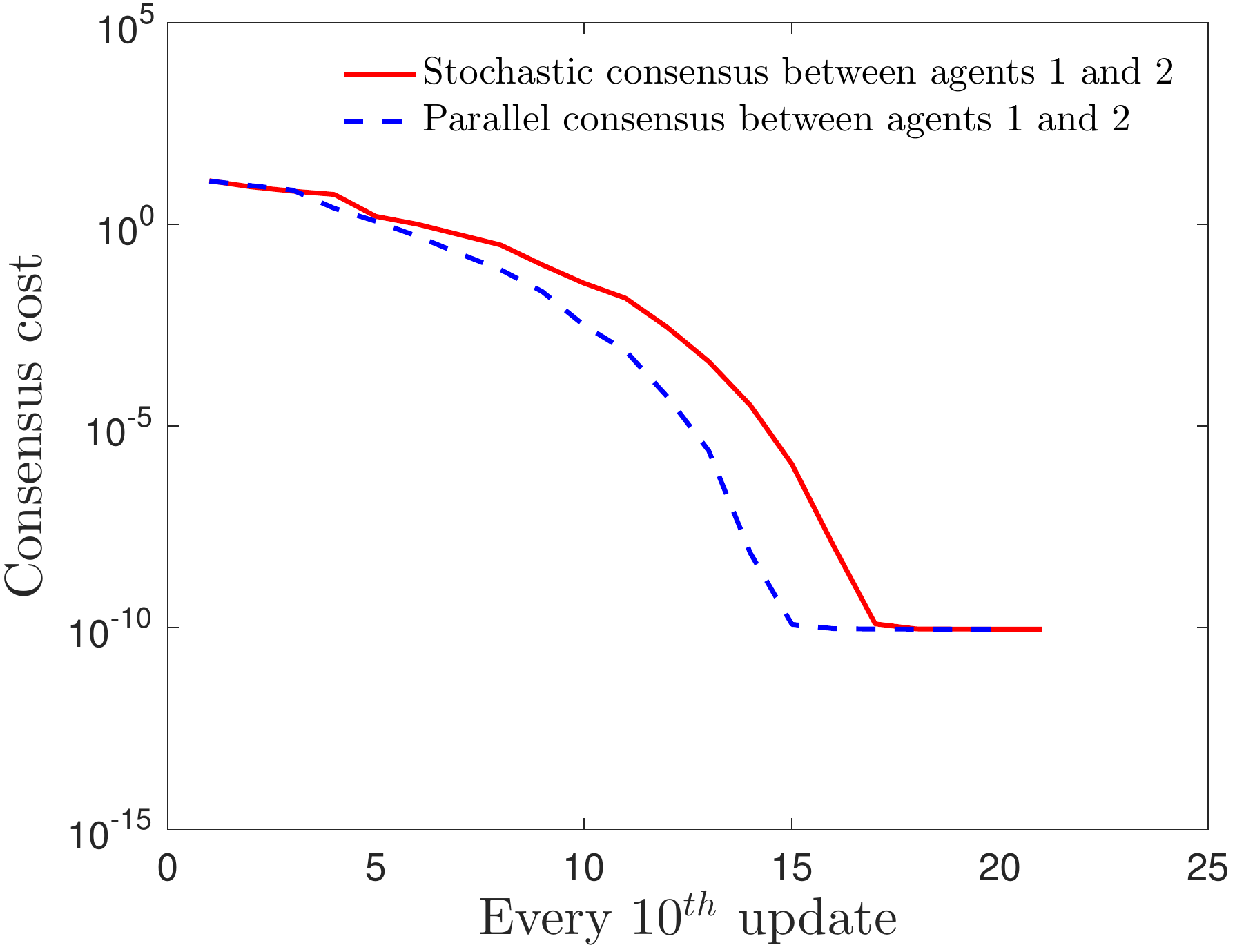}
\justify{(d) Both stochastic and parallel variants achieve consensus between the agents (Case 2).}
 \end{minipage}\\
\begin{minipage}[b]{0.48\hsize}
\centering
\includegraphics[width=\hsize]{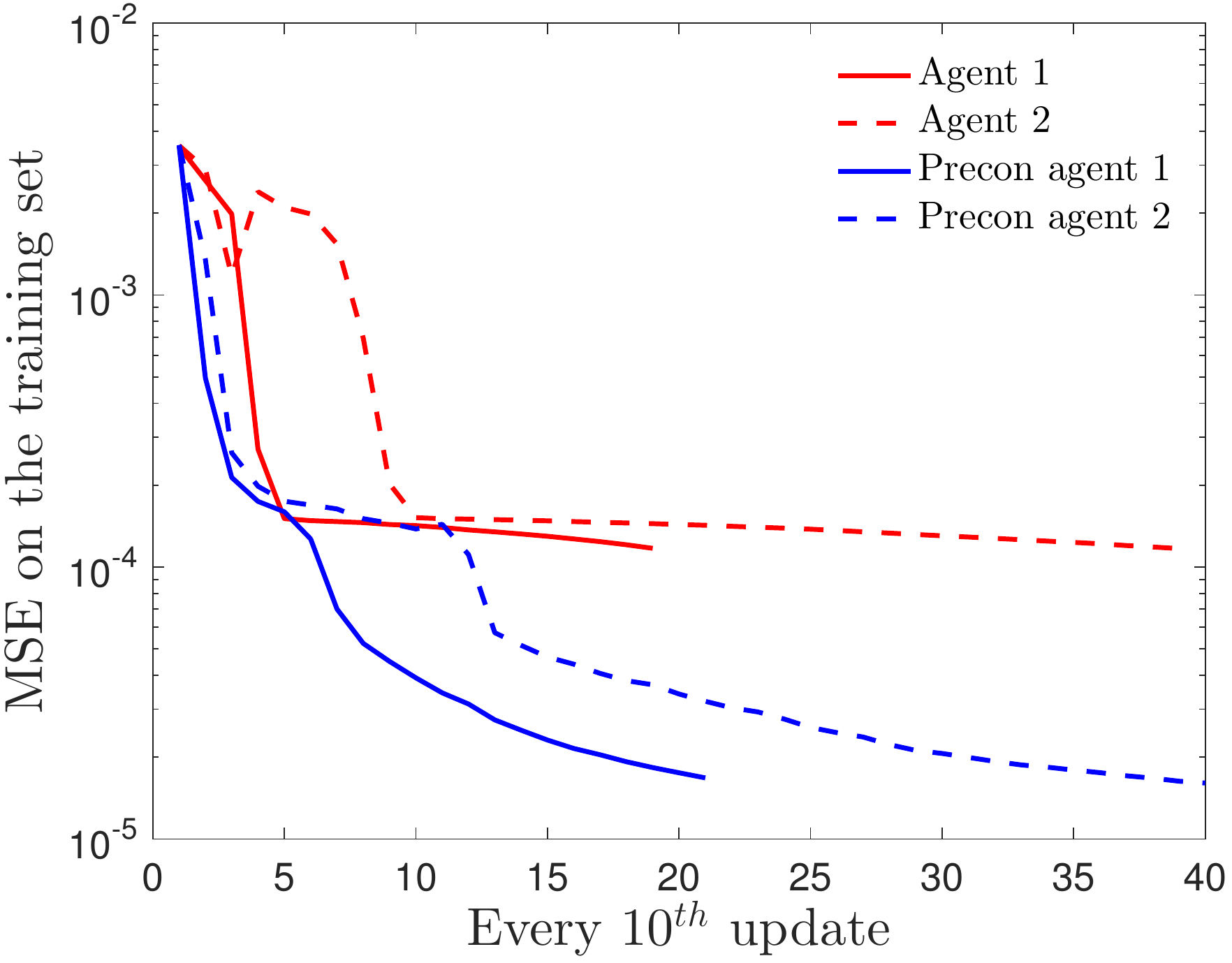}
\justify{(e) Precon Gossip performs better than Stochastic Gossip on ill-conditioned data (Case 3).}
\end{minipage}
\hfill
 \begin{minipage}[b]{0.48\hsize}
\centering
\includegraphics[width=\hsize]{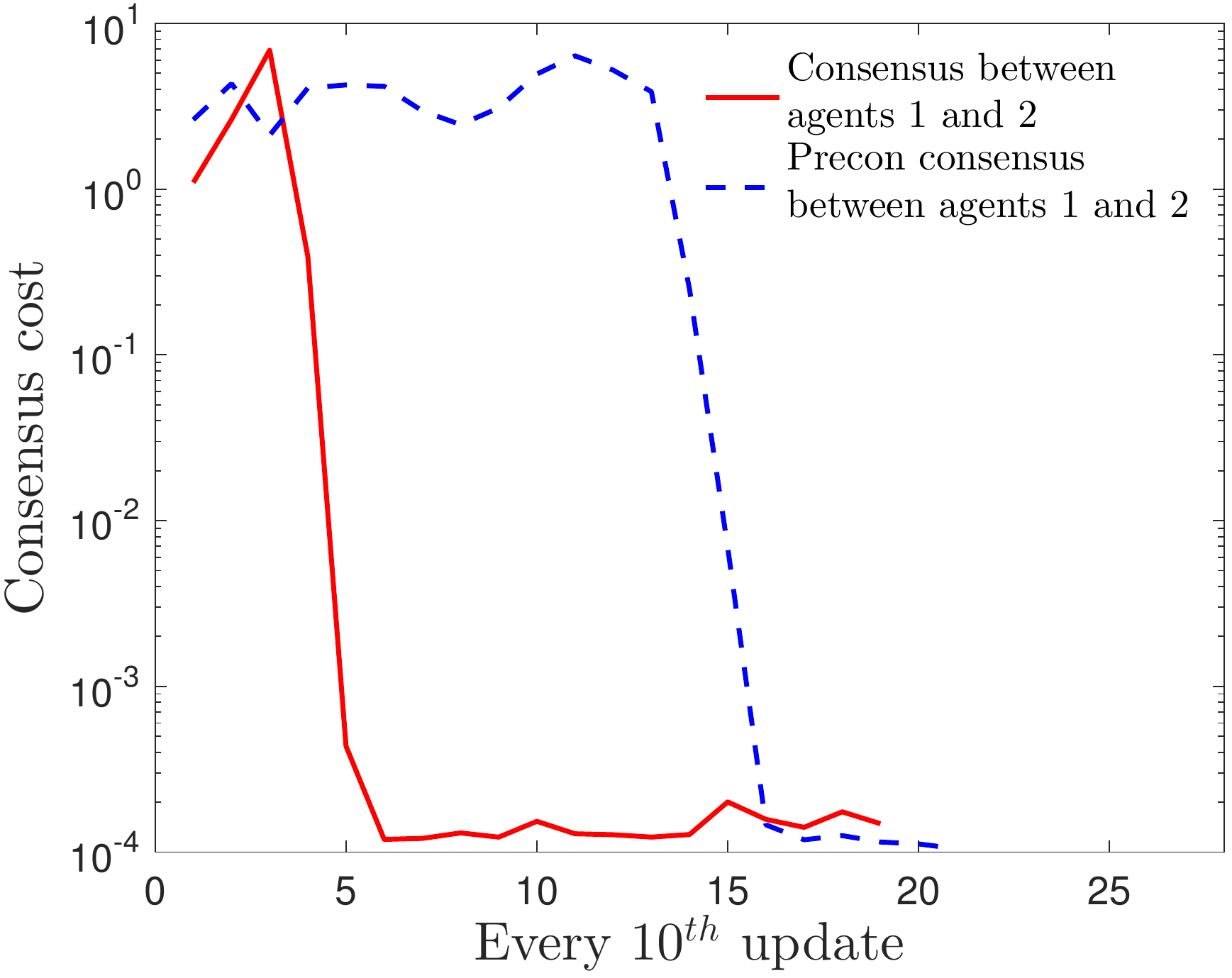}
\justify{(f) Both Precon Gossip and Stochastic Gossip allow agents to reach consensus (Case 3).\newline}
\end{minipage}
\caption{Performance  of the proposed algorithms on low-rank matrix completion problems. Figures (a)\&(b)  correspond to the experimental setup described in Case 1, (c)\&(d) correspond to Case 2, and (e)\&(f)  correspond to Case 3. Figures best viewed in color. }
\label{fig:all}
\end{figure}

\textbf{Case 3: ill-conditioned instances.} We consider a problem instance of size $5\,000 \times 50\,000$ of rank $5$ and impose an exponential decay of singular values with condition number $500$ and OS $6$. Figures \ref{fig:all}(e)\&(f) show the performance of Stochastic Gossip and its preconditioned variant for $\rho = 10^3$. During the initial updates, the preconditioned variant aggressively minimizes the completion term of (\ref{eq:decentralized_grassmann_formulation}), which shows the effect of the preconditioner (\ref{eq:preconditioner}). Eventually, consensus among the agents is achieved.

\begin{figure}
\centering
\begin{minipage}[b]{0.48\hsize}
\centering
\includegraphics[width=\hsize]{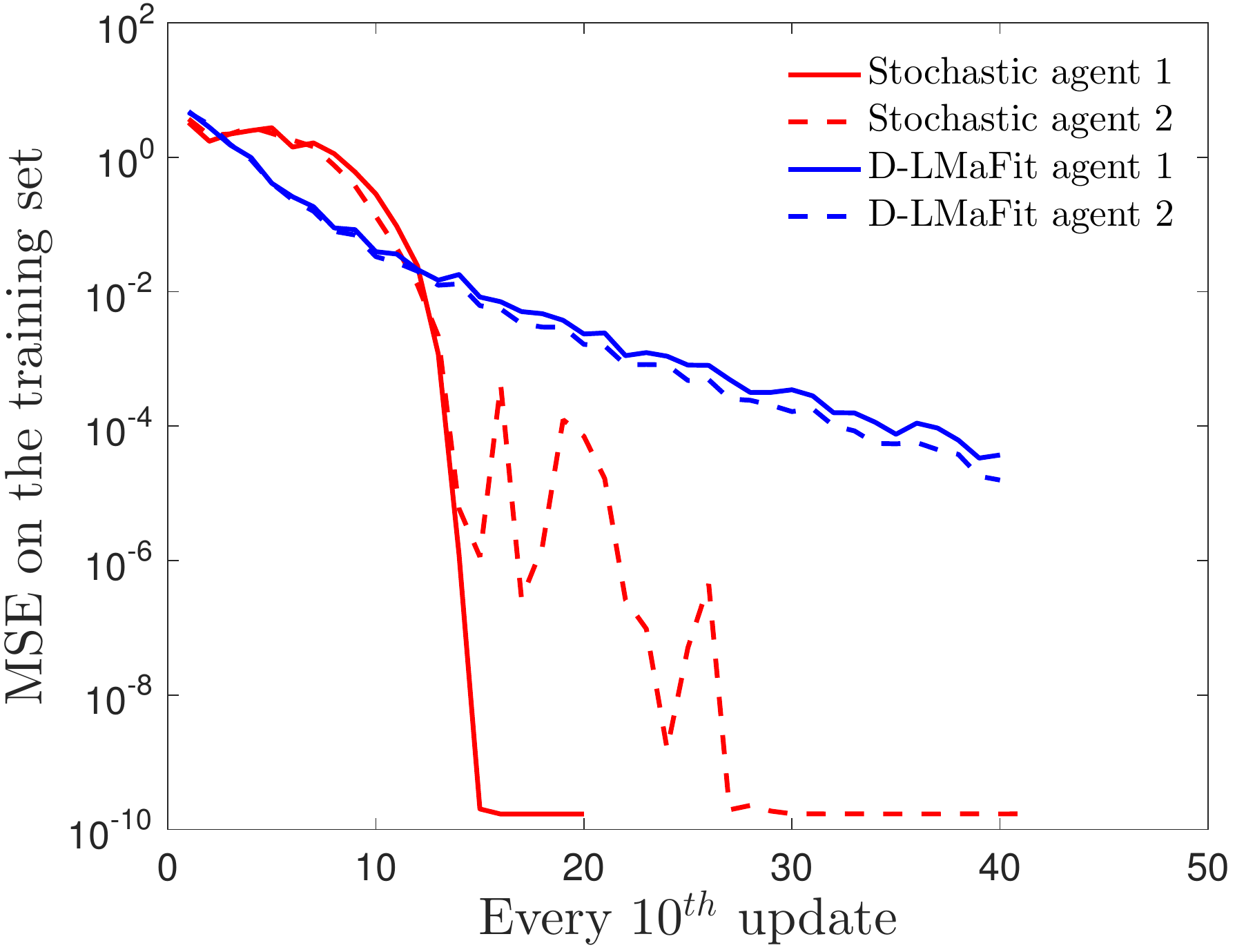}
\justify{(a) Stochastic Gossip outperforms D-LMaFit on solving the local matrix completion tasks.}
\end{minipage}
\hfill
\begin{minipage}[b]{0.48\hsize}
\centering
\includegraphics[width=\hsize]{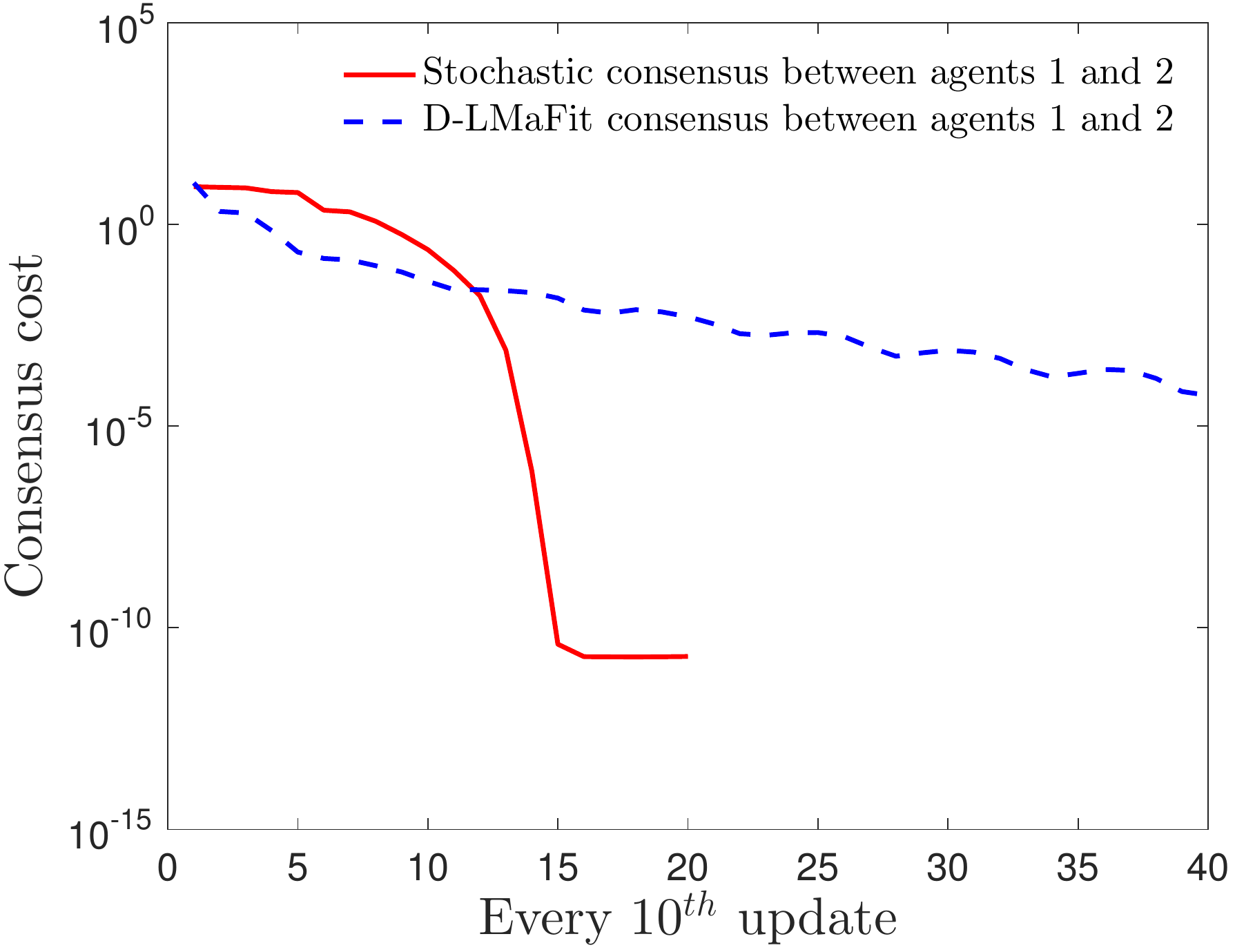}
\justify{(b) Stochastic Gossip achieve faster consensus between the agents than D-LMaFit.\newline }
\end{minipage}
\caption{Comparisons with D-LMaFit (Case 4). Figures best viewed in color. }
\label{fig:all_dlmafit}
\end{figure}

\textbf{Case 4: Comparisons with state-of-the-art.}
We show comparisons with D-LMaFit \citep{ling12a,lin15a}, \changeBMM{the only publicly available decentralized algorithm to the best of our knowledge}. It builds upon the batch matrix completion algorithm in \citep{wen12a} and is adapted to decentralized updates of the low-rank factors. It requires an inexact dynamic consensus step at every iteration by performing an Euclidean average of low-rank factors (of all the agents). In contrast, our algorithms enforce {\it soft averaging} of only two agents at every iteration with the consensus term in (\ref{eq:decentralized_grassmann_formulation}). We employ a smaller problem instance in this experiment since the D-LMaFit code (supplied by its authors) does {\it not} scale to large-scale instances. D-LMaFit is run for $400$ iterations, i.e., each agent performs $400$ updates. We consider a problem instance of size $500\times 12\,000$, rank $5$, and OS $6$. D-LMaFit is run with the default parameters. For Stochastic Gossip, we set $\rho = 10^3$. As shown in Figure \ref{fig:all_dlmafit}, Stochastic Gossip quickly outperforms D-LMaFit. Overall, Stochastic Gossip takes fewer number of updates of the agents to reach a high accuracy.

\begin{table}[t]
    \begin{minipage}[t]{1\linewidth}
      \caption{Mean test RMSE on the Netflix dataset  with different number of agents ($N$) and rank $10$ (Case 5). Our decentralized approach, Stochastic Gossip, is comparable to the state-of-the-art batch algorithm, RTRMC.}\label{tab:netflix} 
      \vspace{8pt}
      \scriptsize 
      \centering
        \begin{tabular}{p{0.9cm} p{0.9cm} p{0.9cm} p{0.9cm} p{0.9cm}  p{2cm}} 
\toprule
 \multicolumn{5}{ c }{Stochastic Gossip}  &  \multicolumn{1}{ c }{Batch method} \\
\midrule
 $N$=2 & $N$=5 & $N$=10 & $N$=15 & $N$=20 &  \multicolumn{1}{ c }{RTRMC} \\
%\hdashline
 $0.877 $ & $ 0.885$ &$ 0.891$ & $0.894$ &  $0.900$ & \multicolumn{1}{ c }{0.873}\\
\bottomrule
\end{tabular}
    \end{minipage}%
    \hspace{0.2cm}
    \begin{minipage}[t]{1\linewidth}
      \centering
        \caption{Mean test RMSE on MovieLens 10M dataset with different number of agents ($N$) and across ranks (Case 5). Our decentralized approach, Stochastic Gossip, is comparable to the state-of-the-art batch algorithm, RTRMC.}\label{tab:movielens} 
        \vspace{8pt}
        \scriptsize 
        \begin{tabular}{ p{3cm} p{1cm} p{1cm} p{1cm} p{1cm}}
%\hline
\toprule
& Rank $3$ & Rank $5$ & Rank $7$ & Rank $9$\\
\midrule
%NMAE on test set & $0.1519 \pm 3\cdot 10^{-3}$ & $ \bf{ 0.1507 \pm 3\cdot 10^{-3}}$ &$0.1531 \pm 2\cdot 10^{-3}$ & $0.1543 \pm 1\cdot 10^{-3}$\\
\noteBM{$N = 10$} &  $0.844$ &$0.836$ &$0.845 $ & $0.860$ \\
%\hdashline
\noteBM{$N=5$} & $0.834$ &  $0.821$  & $0.829$  & $ 0.841$ \\
%\hdashline
%$9$ & $0.1386$ \\
\noteBM{RTRMC} (batch) & $0.829$ & $0.814$  &$0.812$  & $0.814$  \\
% \hline
 \bottomrule
\end{tabular}
    \end{minipage} 
%    \\ \newline
%  \hspace*{6cm}
%    \begin{minipage}[t]{.5\linewidth}
%      \centering
%        \caption{Average test NMSE scores obtained on multitask datasets.} %Stochastic Gossip learns a $r$-dimensional subspace. The NMSE scores obtained by Stochastic Gossip are comparable to the scores Alt-Min, which learns an $m$-dimensional subspace.}
%\label{tab:multitask_datasets} 
%        \scriptsize 
%       \begin{tabular}{ p{1.0cm}|p{0.8cm}|p{0.8cm}|p{0.8cm}|p{0.8cm}| p{1cm}}
%\hline
%Datasets & \multicolumn{4}{|c|}{Stochastic Gossip $N = 6$}  & Alt-Min (batch method) \\
%\hline
%& $r=3$ & $r = 5$ & $r=7$ & $r=9$ & $r = m$ \\
%\hdashline
%Parkinsons &$0.345$  &$0.339$ & $0.342$& $0.341$ & $0.340$ \\
%\hdashline
%School & $0.761$ & $ 0.786$ & $0.782$ &$ 0.786$ &$0.781$ \\
% \hline
%\end{tabular}
%    \end{minipage} 
\end{table}

\textbf{Case 5: comparisons on the Netflix and the MovieLens-10M data-sets.}  
The Netflix dataset (obtained from the code of \citep{recht13a}) consists of $100\,480\,507$ ratings by $480\,189$ users for $17\,770$ movies. We perform $10$ random $80/20$-train/test partitions. The training ratings are centered around $0$, i.e., the mean rating is subtracted. We split both the train and test data among the agents along the number of users. We run Stochastic Gossip with $\rho = 10^7$ (set with cross validation) and for $400(N-1)$ iterations and $N=\{2,5,10,15,20 \}$ agents. We show the results for rank $10$ (the choice is motivated in \citep{boumal15a}). Additionally for Stochastic Gossip, we set the regularization parameter to $\lambda =0.01$. For comparisons, we show the best test root mean square error (RMSE) score obtained by RTRMC \citep{boumal11a,boumal15a}, which is a batch method for solving the matrix completion problem on the Grassmann manifold. RTRMC employs a preconditioned trust-region algorithm. 
\noteBM{In order to compute the test RMSE for Stochastic Gossip on the full test set (not on the agent-partitioned test sets), we use the (\emph{Fr\'echet}) mean subspace of the subspaces obtained by the agents as the final subspace obtained by our algorithm.} Table \ref{tab:netflix} shows the RMSE scores for Stochastic Gossip and RTRMC averaged over ten runs. Table \ref{tab:netflix} shows that the proposed gossip approach allows to reach a reasonably good solution on the Netflix data with different number of agents (which interact minimally among themselves). \noteBM{It should be noted that as the number of agents increases, the consensus problem (\ref{eq:decentralized_grassmann_formulation}) becomes challenging. Similarly, consensus of agents at higher ranks is more challenging as we need to learn a larger subspace.} Figure \ref{fig:all_netflix} shows the consensus of agents for the case $N=10$. 

We also show the results on the MovieLens-10M dataset of $10\, 000\, 054$ ratings by $71\,567$ users for $10\,677$ movies \citep{movielens}. The setup is similar to the earlier Netflix case. We run Stochastic gossip with $N=\{5,10 \}$ and $\rho = 10^5$. We show the RMSE scores for different ranks in Table \ref{tab:movielens}.

%\begin{figure}[t]
%\centering
%\begin{tabular}{cc}
%\noindent \begin{minipage}[b]{0.5\hsize}
%\centering
%\includegraphics[width=\hsize]{figures/netflix/consensus_new.pdf}
%%{\scriptsize Netflix: Stochastic Gossip achieves consensus of agents.}
%\end{minipage}
%%\noindent \begin{minipage}[b]{0.25\hsize}
%%\centering
%%\includegraphics[width=\hsize]{figures/multitask/distancetosubspace/consensus.eps}
%%{\scriptsize (g) Agents converge to the optimal  $5$-dimensional subspace $\mat{U}_*$.}
%%%{\small (a)}
%%\end{minipage}
%%\noindent \begin{minipage}[b]{0.25\hsize}
%%\centering
%%\includegraphics[width=\hsize]{figures/multitask/parkinson/nmse.eps}
%%{\scriptsize (h) Agents learn a $3$-dim. subspace with comparable NMSE of AltMin.}
%%%{\small (b)}
%%\end{minipage}
%\end{tabular}
%\caption{Netflix: Stochastic Gossip achieves consensus of agents (Case 5). Figure best viewed in color. }
%\label{fig:all_netflix}
%\end{figure}

\begin{figure}
\centering
\begin{minipage}[b]{0.5\hsize}
\centering
\includegraphics[width=\hsize]{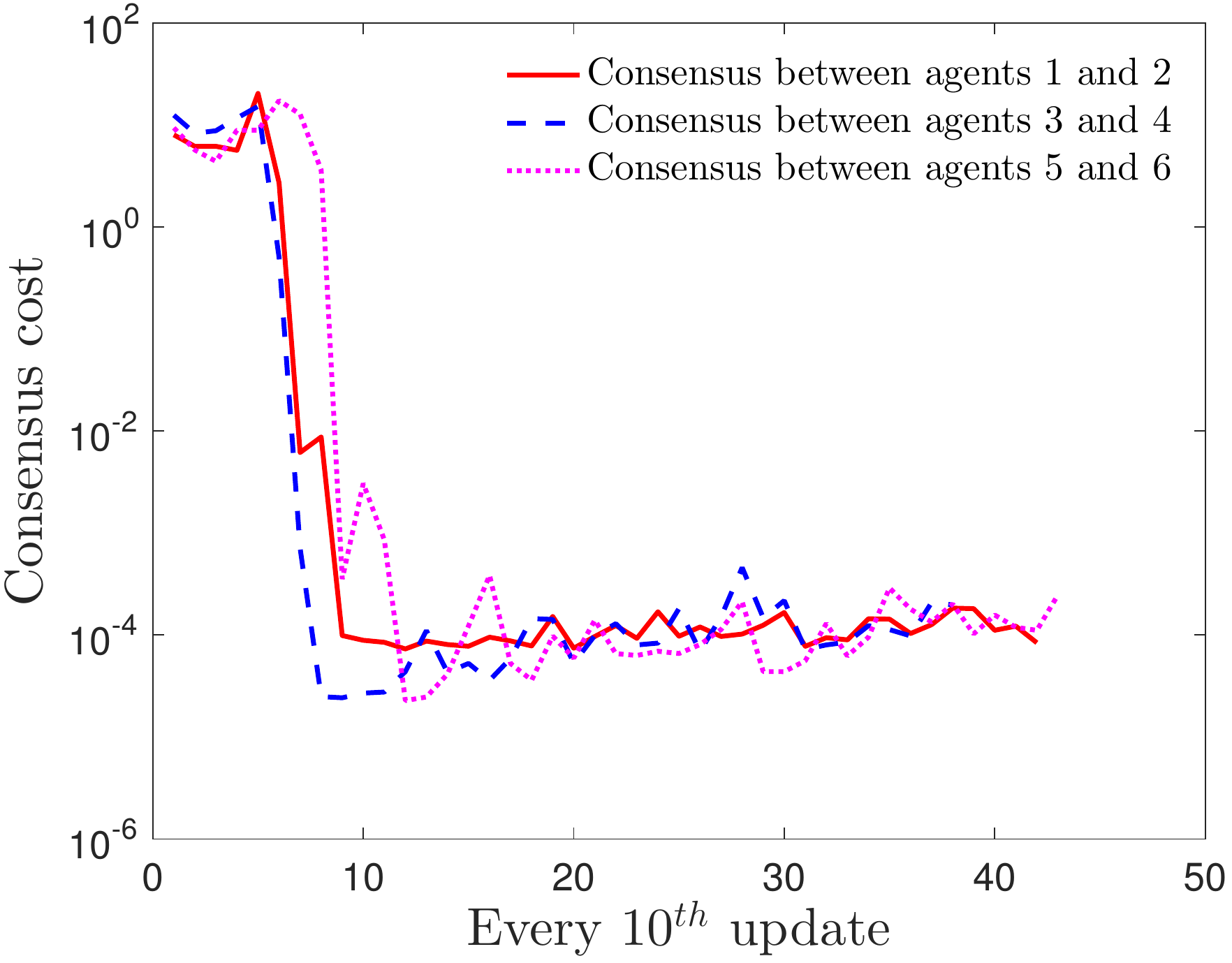}
%{\scriptsize Netflix: Stochastic Gossip achieves consensus of agents.}
\end{minipage}
\caption{Matrix completion experiment on the Netflix dataset (Case 5). Our decentralized approach, Stochastic Gossip, achieves consensus between the agents. Figure best viewed in color. }
\label{fig:all_netflix}
\end{figure}

\subsection{Multitask comparisons}
In this section, we discuss the numerical results on the low-dimensional multitask feature learning problem (\ref{eq:grassmann_formulation_mtl}) on different benchmarks. The regularization parameter $\lambda$ that is used to solve for $w_t$ in (\ref{eq:grassmann_formulation_mtl}) is set to $\lambda = 0$ for Case 6 and is set to $\lambda = 0.1$ for Case 7.

\textbf{Case 6: synthetic datasets.} We consider a toy problem instance with $T=1000$ tasks. The number of training instance in each task $t$ is between $10$ and $50$ ($d_t$ chosen randomly). The input space dimension is $m =100$. The training instances $\mat{X}_t$ are generated according to the Gaussian distribution with zero mean and unit standard deviation. A $5$-dimensional feature subspace $\mat{U}_*$ for the problem instance is generated as a random point on $\mat{U}_* \in \Stiefel{5}{100}$. The weight vector $w_t$ for the task $t$ is generated from the Gaussian distribution with zero mean and unit standard deviation. The labels for training instances for task $t$ are computed as $y_t = \mat{X}_t \mat{U}_*\mat{U}_*{^\top}w_t$. The labels $y_t$ are subsequently perturbed with a random mean zero Gaussian noise with $10^{-6}$ standard deviation. The tasks are uniformly divided among $N=6$ agents and Stochastic Gossip is initialized with $r=5$ and $\rho = 10^3$. 

Figure \ref{fig:all_multitask}(a) shows that all the agents are able to converge to the optimal subspace~$\mat{U}_*$.

%\begin{figure}[t]
%\begin{tabular}{cc}
%\hspace*{-0.6cm}
%\noindent 
%\begin{minipage}[b]{0.5\hsize}
%\centering
%\includegraphics[width=\hsize]{figures/multitask/distancetosubspace/consensus_new.pdf}
%{\scriptsize (a) Agents converge to the optimal  $5$-dimensional subspace $\mat{U}_*$ on a synthetic dataset (Case 6).}
%%{\small (a)}
%\end{minipage}
%%\hspace*{-0.6cm}
%\noindent \begin{minipage}[b]{0.5\hsize}
%\centering
%\includegraphics[width=\hsize]{figures/multitask/parkinson/nmse_new.pdf}
%{\scriptsize (b) Agents learn a $3$-dim. subspace with comparable NMSE to Alt-Min \citep{argyriou08a} (Case 7).}
%%{\small (b)}
%\end{minipage}
%\end{tabular}
%\caption{Comparisons on multitask learning benchmarks. Figures best viewed in color. }
%\label{fig:all_multitask}
%\end{figure}

\begin{figure}
\centering
\begin{minipage}[b]{0.48\hsize}
\centering
\includegraphics[width=\hsize]{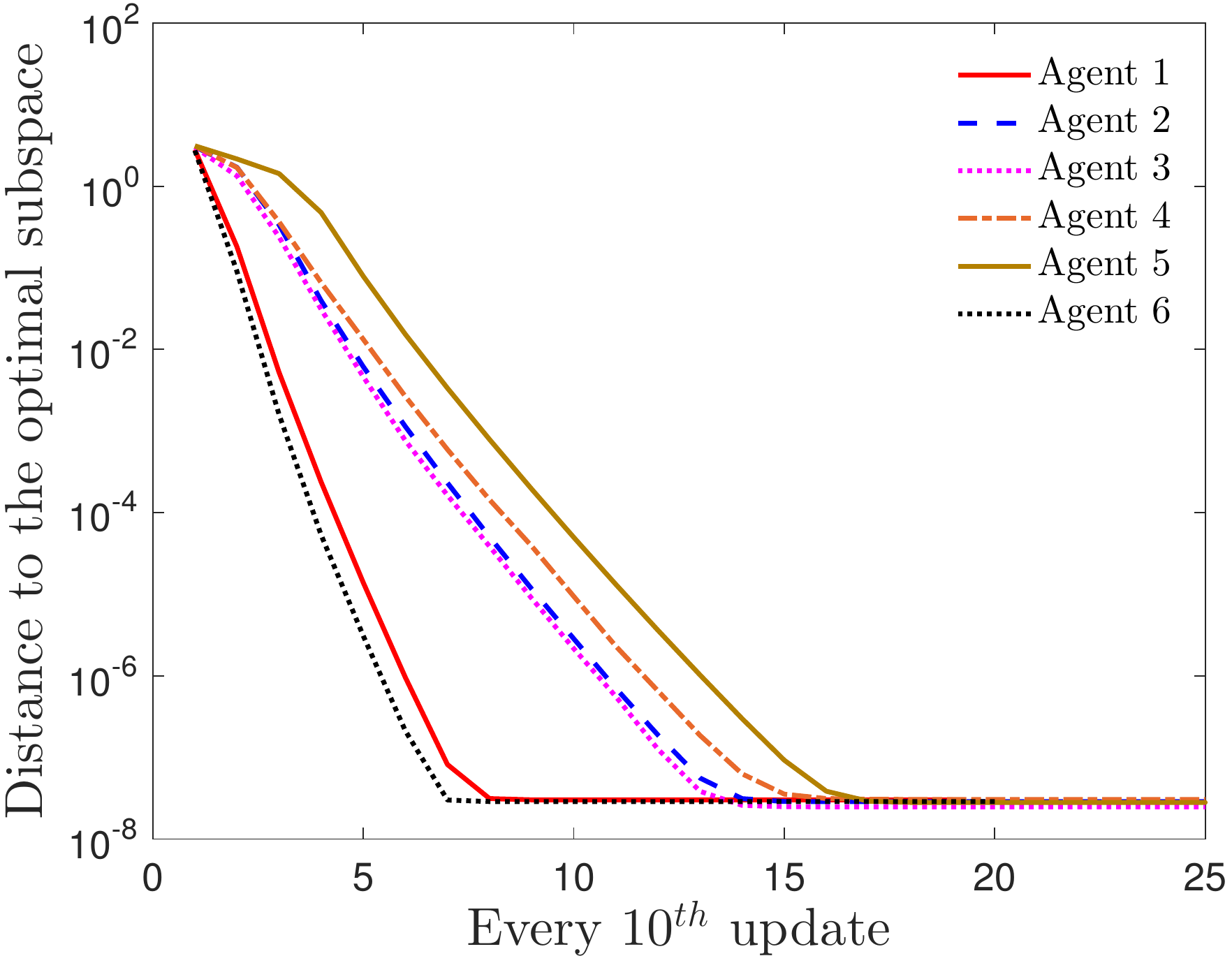}
\justify{(a) Agents converge to the optimal  $5$-dimensional subspace $\mat{U}_*$ on a synthetic dataset (Case 6).}
\end{minipage}
\hfill
\begin{minipage}[b]{0.48\hsize}
\centering
\includegraphics[width=\hsize]{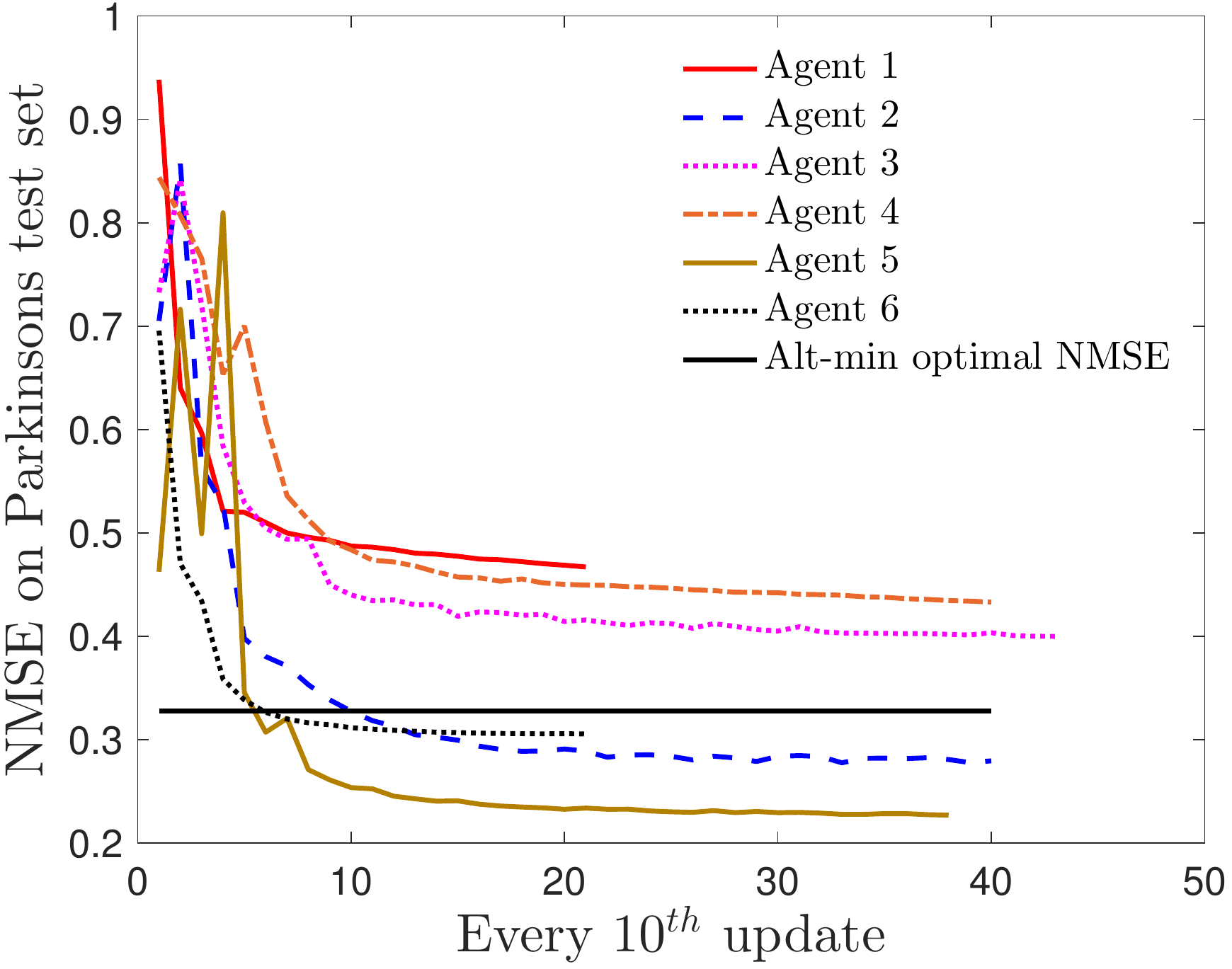}
\justify{(b) Agents learn a $3$-dim. subspace with comparable NMSE to Alt-Min \citep{argyriou08a} (Case 7).}
\end{minipage}
\caption{Comparisons on multitask learning benchmarks. Figures best viewed in color. }
\label{fig:all_multitask}
\end{figure}

\textbf{Case 7: comparisons on multitask benchmarks.} We compare the generalization performance with formulation (\ref{eq:grassmann_formulation_mtl}) solved by the proposed gossip algorithm against state-of-the-art multitask feature learning formulation (\ref{eq:formulation_mtl}) proposed in \citep{argyriou08a}. Argyriou et al. \citep{argyriou08a} propose an alternate minimization batch algorithm (Alt-Min) to solve an equivalent convex problem of (\ref{eq:formulation_mtl}). Conceptually, Alt-Min alternates between the subspace learning step and task weight vector learning step. Alt-Min does optimization over an $m\times m$-dimensional space. In contrast, we learn a low-dimensional $m\times r$ subspace in (\ref{eq:grassmann_formulation_mtl}), where $r\leq m$. As discussed below, the experiments show that our algorithms obtain a competitive performance even for values of $r$ where $r < m$, thereby making the formulation (\ref{eq:grassmann_formulation_mtl}) suitable for low-rank multitask feature learning. 
%a viable alternative to (\ref{eq:formulation_mtl}). 
%\changePJ{Should we discuss complexity when we are not doing any timing experiments? There may be better algorithms than Alt-Min.}

We compare Stochastic Gossip and Alt-Min on two real-world multitask benchmark datasets: Parkinsons and School. 
In the \emph{Parkinsons} dataset, the goal is to predict the Parkinson's disease symptom score at different times of $42$ patients with $m = 19$ bio-medical features \citep{frank10a,mjaw12,muandet13a}. A total of $5\,875$ observations are available. The symptom score prediction problem for each patient is considered as a task ($T=42$). The \emph{School} dataset consists of $15\,362$ students from $139$ schools \citep{goldstein91a,Evgeniou05,argyriou08a}. The aim is to predict the performance (examination score) of the students from the schools, given the description of the schools and past record of the students. A total of $m =28$ features are given. The examination score prediction problem for each school is considered as a task ($T=139$).

We perform $10$ random $80/20$-train/test partitions. We run Stochastic Gossip with $\rho = 10^6$, $N=6$, and for $200(N-1)$ iterations. Alt-Min is run till the relative change in its objective function (across consecutive iterations)  is below the value $10^{-8}$. Following~\citep{argyriou08a,mjaw11a,chen11a}, we report the performance of multitask algorithms in terms of {\it normalized mean squared error} (NMSE). It is defined as the ratio of the mean squared error (MSE) and the variance of the label vector. 

Table \ref{tab:multitask_datasets} shows the NMSE scores (averaged over all $T$ tasks and ten runs) for both the algorithms. The comparisons on benchmark multitask learning datasets show that we are able to obtain {smaller} NMSE: $0.339$ (Parkinsons, $r$=5) and $0.761$ (School, $r$=3). We also obtain these NMSE at much a smaller rank compared to Alt-Min algorithm. 

Figure \ref{fig:all_multitask}(b) shows the NMSE scores obtained by different agents, where certain agents outperform Alt-Min. Overall, the average performance across the agents matches that of the  batch Alt-Min algorithm. %\changePJ{No discussion is there for figures}. \changePJ{Is it ok to focus on batch vs decentralized performance in this discussion? Should we say that matching batch performance itself is a good thing?}

%Stochastic Gossip obtains competitive NMSE scores on lower dimensional subspaces than Alt-Min. 

\begin{table}[t]
\caption{Mean test NMSE scores obtained on multitask datasets across different ranks $r$ (Case 7). The search space of our decentralized approach (Stochastic Gossip) is $m\times r$ while that of the batch algorithm Alt-Min is $m \times m$. The generalization performance of Stochastic Gossip is comparable to Alt-Min.
} 
%Stochastic Gossip learns a $r$-dimensional subspace. The NMSE scores obtained by Stochastic Gossip are comparable to the scores Alt-Min, which learns an $m$-dimensional subspace.}
\label{tab:multitask_datasets} 
\begin{center} \scriptsize % \small
\begin{tabular}{ p{3cm} p{0.9cm} p{0.9cm} p{0.9cm} p{0.9cm} p{3cm}}
\toprule
\multirow{2}{*}{Datasets} & \multicolumn{4}{c}{Stochastic Gossip with $N$=$6$}  & \multirow{2}{*}{Alt-Min (batch)} \\
& $r$=$3$ & $r$=$5$ & $r$=$7$ & $r$=$9$ &   \\
\midrule
Parkinsons  ($m$=$19$) &$0.345$  &$0.339$ & $0.342$& $0.341$ & $0.340$ \\
School  ($m$=$28$) & $0.761$ & $ 0.786$ & $0.782$ &$ 0.786$ &$0.781$ \\
\bottomrule
\end{tabular}
\end{center} 
\end{table}

\section{Conclusion}
We have proposed a decentralized Riemannian gossip approach to subspace learning problems. The sub-problems are distributed among a number of agents, which are then required to achieve consensus on the global subspace. Building upon the non-linear gossip framework, we modeled this as minimizing a weighted sum of \emph{task solving} and \emph{consensus} terms on the Grassmann manifold. The consensus term exploits the rich geometry of the Grassmann manifold, which allows to propose a novel stochastic gradient algorithm for the problem with simple updates. Experiments on two interesting applications -- low-rank matrix completion and multitask feature learning -- show the efficacy of the proposed Riemannian gossip approach. Our experiments demonstrate  the benefit of exploiting the geometry of the search space that arise in subspace learning problems.

Currently in our gossip framework setup, the agents are tied with a single learning stepsize sequence, which is akin to working with a single universal clock. As future research direction, we intend to work on decoupling the learning rates used by different agents \citep{colin16a}.

\small

%\begin{acknowledgements}
%\end{acknowledgements}

%\bibliographystyle{spbasic}
\bibliographystyle{plainnat}
\setcitestyle{authoryear,open={((},close={))}}
%\bibpunct{(}{)}{;}{a}{,}{,}

\bibliography{mygossipref}

\end{document}